\pdfoutput=1

\documentclass[twoside]{article}

 \usepackage[accepted]{aistats2020}
\usepackage[round]{natbib}

\usepackage{float}
\usepackage[caption=false,font=footnotesize]{subfig}
\usepackage{multirow}
\usepackage[pdftex]{graphicx}
\usepackage{amsmath}
\usepackage{amssymb}
\usepackage{amsthm}
\usepackage{color}
\usepackage{makecell}
\usepackage{mathtools}
\usepackage{microtype}
\usepackage[cmintegrals]{newtxmath}
\usepackage[ruled]{algorithm}
\usepackage{algpseudocode}
\usepackage{IEEEtrantools}

\newtheorem{theorem}{Theorem}

\newtheorem{lemma}{Lemma}
\newtheorem{proposition}{Proposition}

\newtheorem{assumption}{Assumption}

\makeatletter
\newcommand{\algmargin}{\the\ALG@thistlm}
\makeatother
\newlength{\whilewidth}
\algdef{SE}[parWHILE]{parWhile}{EndparWhile}[1]
  {\parbox[t]{\dimexpr\linewidth-\algmargin}{%
     \hangindent\whilewidth\strut\algorithmicwhile\ #1\ \algorithmicdo\strut}}{\algorithmicend\ \algorithmicwhile}%
\algnewcommand{\parState}[1]{\State%
  \parbox[t]{\dimexpr\linewidth-\algmargin}{\strut #1\strut}}

\newcommand{\maximize}{\mathop{\rm maximize}}
\newcommand{\minimize}{\mathop{\rm minimize}}
\newcommand{\subjecto}{\mathop{\rm subject\;to}}

\newcommand{\argmax}{\mathop{\rm argmax}}
\newcommand{\argmin}{\mathop{\rm argmin}}

\newcommand{\bpsi}{\ensuremath{{\boldsymbol \uppsi}}}

\begin{document}

% If your paper is accepted and the title of your paper is very long,
% the style will print as headings an error message. Use the following
% command to supply a shorter title of your paper so that it can be
% used as headings.
%
%\runningtitle{I use this title instead because the last one was very long}

% If your paper is accepted and the number of authors is large, the
% style will print as headings an error message. Use the following
% command to supply a shorter version of the authors names so that
% they can be used as headings (for example, use only the surnames)
%
%\runningauthor{Surname 1, Surname 2, Surname 3, ...., Surname n}

\twocolumn[

\aistatstitle{Contextual Constrained Learning for Dose-Finding Clinical Trials}

\aistatsauthor{ Hyun-Suk Lee \And Cong Shen \And  James Jordon \And Mihaela van der Schaar}

\aistatsaddress{ University of Cambridge\\Yonsei University \And  University of Virginia \And University of Oxford \And University of Cambridge\\UCLA} ]

\begin{abstract}
Clinical trials in the medical domain are constrained by budgets. The number of patients that can be recruited is therefore limited. When a patient population is heterogeneous, this creates difficulties in learning subgroup specific responses to a particular drug and especially for a variety of dosages. In addition, patient recruitment can be difficult by the fact that clinical trials do not aim to provide a benefit to any given patient in the trial.
In this paper, we propose C3T-Budget, a contextual constrained clinical trial algorithm for dose-finding under both budget and safety constraints. The algorithm aims to maximize drug efficacy {\em within} the clinical trial while {\em also} learning about the drug being tested. %The trade-off between drug efficacy and learning must be balanced and considerations of the trial budget must also be weighed carefully.
C3T-Budget recruits patients with consideration of the remaining budget, the remaining time, and the characteristics of each group, such as the population distribution, estimated expected efficacy, and estimation credibility. In addition, the algorithm aims to avoid unsafe dosages. These characteristics are further illustrated in a simulated clinical trial study, which corroborates the theoretical analysis and demonstrates an efficient budget usage as well as a balanced learning-treatment trade-off. %By designing clinical trials using C3T-Budget, patient enrollment is made more appealing, allowing for a more thorough clinical trial overall.
%When the budget is very limited, we propose a modified version of C3T-Budget which focuses on the groups with high efficacies, allowing discovery of subgroups for which a drug is effective.
%Through the experiments, we show that our algorithms effectively uses the budget to balance the efficacy-learning trade-off.
\end{abstract}

\section{INTRODUCTION}
Traditionally, dose-finding studies are early stage clinical trials used to evaluate the safety of a new treatment, and are often referred to as Phase I clinical trials.
Thus, traditional dose-finding methods aim to find the most appropriate dose for use in further phases of clinical trials by considering only toxicity of the treatment \citep{storer1989design,o1990continual}.
{\color{black} Recently, however, both toxicity and efficacy have been considered in dose-finding methods to accelerate the development process of new treatments and to reduce costs \citep{zang2014adaptive}.
As the efficacy is considered in dose-finding,
the dose-finding methods should be designed for both \textit{clinical research} (i.e., to find the appropriate dose) and \textit{clinical practice} (i.e., to benefit the patients participating) \citep{berry2004bayesian}.
This gives rise to an inevitable trade-off between information gathering and treatment effectiveness, which makes the dose-finding problem more complex.}

Heterogeneous groups of patients also make the dose-finding problem more complex \citep{wages2015phase}.
In many cases, the efficacy and toxicity of a dose will vary between groups as shown in Table \ref{table:trial_examples}.
Thus, with the heterogeneous groups, the goal of dose-finding clinical trials is to find the most appropriate dose for each group.
On the other hand, the presence of heterogeneous responses can also be a blessing -- a drug need only be effective for a single subgroup for it to be approved.
Furthermore, in general, different groups will be of different sizes (i.e. a different number of people belong to each group on a population level), which may result in inaccurate dose-finding for under-represented groups \citep{hussain2004ethnic}.
To avoid this, an algorithm for dose-finding should account for differing group sizes (which manifest themselves in terms of patient enrolment/arrival rates).

\begin{table*}[ht]
\vspace{-0.17in}
\caption{Examples of Studies with Heterogeneous Groups in the Literature} \label{table:trial_examples}
\scriptsize
\begin{center}
\begin{tabular}{c|c|c|c}
\hline
\textbf{Study} & \textbf{Treatment} & \textbf{Groups} & \textbf{Results} \\
\hline
\citep{polyzos2008activity} & Sunitinib & Type of cancer & \makecell{37\% RR, 48\% SD for renal cell carcinoma,\\
8\% RR, 70\% SD for gastrointestinal stromal tumors,\\9\% RR, 19\% SD for lung cancer\\(RR: response rate, SD: stable disease)} \\
\hline
\citep{kim2009ugt1a1} & Irinotecan & \makecell{3 Groups according to\\genetic characteristics} & \makecell{350 mg/m$^2$ for group 0 and 1\\200 mg/m$^2$ for group 2} \\
\hline
\citep{dasari2013phase} & \makecell{Sorafenib 400 mg bid\\with vorinostat 300 mg daily\\days 1--14 of a 21-day cycle} & Type of cancer & \makecell{Treatment is not tolerable in renal cell carcinoma\\ and non-small cell lung cancer patients\\while tolerable in other cancer types} \\
\hline
\citep{moss2015efficacy} & Ivacaftor & Adult and children & \makecell{Treatment significantly improves lung function\\only in adult patients} \\
\hline
\end{tabular}
\end{center}
\end{table*}

{\color{black} In clinical trials, there are several ethical constraints involved with treating human subjects \citep{o1990continual} such as a limited number of patients and a careful consideration for the safety of the patients.
These constraints highly motivate a clinical trial that effectively utilizes the limited budget for successfully balancing the aforementioned trade-off 
%between clinical research and clinical practice 
while avoiding unsafe doses.
It is worth noting that such a clinical trial might be more ethical than the traditional methods since it more effectively uses the budget and exposes fewer participants to unsafe doses \citep{park2018critical,pallmann2018adaptive}.
In particular, with a joint consideration of heterogeneous groups and limited budget, the relationship between clinical research and clinical practice should be carefully studied.}
From the point of view of \textit{research}, a successful clinical trial with a given limited budget should concentrate on the groups for which the estimate of the optimal dose is poor.
However, a new treatment may not be effective for all groups as shown in Table \ref{table:trial_examples}.
Thus, from the point of view of \textit{practice}, precious budget is being wasted if the trial concentrates on improving estimates for groups for which the new treatment is ineffective.
{\color{black}Moreover, the differing populations of each group should also be considered while balancing the trade-off, since what can be achieved, in terms of both research and practice, depends on the population \citep{hussain2004ethnic}.}
%This relationship between research and practice in clinical trials with a limited budget highly motivates a clinical trial that effectively utilizes the limited budget for successfully balancing these two goals.
\looseness=-1

In this paper, we study a contextual constrained clinical trial (C3T) problem for dose-finding with both budget and safety constraints.
The goal is to effectively balance the aforementioned trade-off while avoiding unsafe doses.
In the C3T problem, patients arrive sequentially (with arrival rates for different groups assumed to be different), and an agent that administrates the clinical trial determines how a new patient will be treated.
In other words, in the C3T problem, the agent's first decision is whether or not the patient is treated (i.e. accepted into the trial), and if so, the agent should then determine the dose that will be allocated to the patient.
\looseness=-1

To answer these two questions, we propose a novel contextual multi-armed bandit (MAB) algorithm for C3Ts with budget constraint called \textit{C3T-Budget}.
{\color{black}In the algorithm, different patient groups are modeled as the context.
In each round, to maximize the effectiveness of \textit{clinical practice}, the dose for each group is chosen to maximize the expected efficacy while satisfying the safety constraint based on information learned from previously treated patients.
Then, given the chosen dose, the algorithm determines whether the patient is treated or not with the aim of maximizing the information from \textit{clinical research}.}
Specifically, we provide a linear programming (LP) approximation to the C3T problem based on the optimal dose, remaining budget, remaining time, and population of the groups.
With the LP approximation, we can determine whether to treat a patient or not by maximizing the accuracy improvement of the efficacy estimation, for which a Bayesian approach is adopted.
C3T-Budget can, therefore, systematically balance the trade-off between treatment and learning by recruiting patients from appropriate subgroups.
However, when the budget is very limited, it is hard to balance the trade-off for all subgroups.
In this case, clinical trials should focus on subgroups with high efficacies since an attempt to focus on all subgroups may lead to a completely failed trial.
To address this, we propose a modified version of C3T-Budget called \textit{C3T-Budget-E}, which concentrates on the subgroups with estimated high efficacies.
%In C3T-Budget-E, the LP approximation in C3T-Budget is modified to make the algorithm concentrate on the subgroups have the high efficacies.
Through numerical results, we show that our algorithms outperform state of the art in terms of efficacy, safety, and recommendation errors.
\looseness=-1

\subsection{Related Works}

%To address the complex problem of dose-finding, multi-armed bandit (MAB) models have been widely studied for use in clinical trials \citep{villar2015multi,varatharajah2018contextual,aziz2019multi}.
%Within the context of MAB, the tradeoff between \textit{research} and \textit{practice} can be seen as the well-known tradeoff between \textit{exploration} and \textit{exploitation}, which has been very well-studied.
%Thus, MAB models are particularly appropriate for clinical trials. In the presence of heterogeneous groups of patients, a contextual MAB is an obvious choice as it allows us to treat the different groups as context for the contextual bandit \citep{varatharajah2018contextual}.

%Since the MAB model was first introduced as a simple clinical trial model in \cite{thompson1933likelihood},
%MAB models have been widely applied in a variety of fields \citep{bouneffouf2019survey}.
{\color{black} MAB models are particularly appropriate for clinical trials since the trade-off between \textit{clinical research} and \textit{clinical practice} can be seen as the well-known tradeoff between \textit{exploration} and \textit{exploitation} in the context of MAB, which has been very well-studied.
Under this pretext,
%with the emergence of new demands on clinical trials such as an acceleration of the development process and more cost-efficient trials,
the application of novel MAB models tailored to clinical trials has been widely studied.}
In \cite{villar2015multi}, a patient allocation strategy for clinical trials is proposed based on the Gittins index.
\cite{villar2018covariate} extend the strategy as the modified forward looking Gittins index to utilize the covariate information.
In \cite{aboutalebi2019learning}, a regret model that considers the safety of dosages is proposed.
Based on the regret model, a dose-finding algorithm is developed, and the regret model helps the algorithm avoid unsafe dosages when finding the optimal dosage.
In \cite{aziz2019multi}, a Thompson sampling-based Phase I clinical trial algorithm is proposed.
In the algorithm, both efficacy and toxicity are considered when finding the optimal dosage.
However, none of these algorithms consider heterogeneous groups of patients.

As the interest in heterogeneous groups of patients in clinical trials has rapidly grown, clinical trial designs that consider heterogeneous groups have begun to see some attention.
In \cite{wages2015phase}, an adaptive clinical trial design for heterogeneous groups was proposed by extending a well-known continual reassessment method (CRM).
However, this work does not consider a limited budget.
Another adaptive clinical trial algorithm with heterogeneous groups is developed in \cite{atan2019sequential}, which is based on the knowledge gradient policy.
In this work, the patient recruitment is determined by considering the budget, but all groups are assumed to have the same arrival distribution.
Moreover, the primary goal is to label the effectiveness of treatments, not dose-finding.
To the best of our knowledge,  \cite{varatharajah2018contextual} is the only work using contextual MAB for clinical trials with heterogeneous groups, and is is probably the closest to our work. 
%In \cite{varatharajah2018contextual}, a clinical trial algorithm with heterogeneous groups is proposed based on contextual MAB which is probably the closest work to ours.
In the algorithm, a simple contextual MAB model is applied to clinical trials by considering groups as contexts.
%To the best of our knowledge, this work is the only work using contextual MAB for clinical trials with heterogeneous groups.
However, the algorithm does not address several of the challenges associated with heterogeneous groups such as the differing populations of the groups and the limited budget.
\looseness=-1

\section{CONTEXTUAL CONSTRAINED CLINICAL TRIAL (C3T) MODEL}
We consider a contextual constrained clinical trial (C3T) for dose-finding with $S$ subgroups (contexts) over a time-horizon, $T$, and limited budget, $B$, indicating the maximum number of patients that can be admitted into the trial.
The set of subgroups is defined as $\mathcal{S}=\{1,2,...,S\}$.
We consider $K$ candidate doses, with the set of candidate doses given by $\mathcal{K}=\{1,2,...,K\}$.
The efficacy, $X_{s,k}$, and toxicity, $Y_{s,k}$, of dose $k$ for subgroup $s$\footnote{For ease of exposition, we use ``subgroup $s$'' and ``patient in subgroup $s$'' interchangeably in the rest of the paper.} are modelled as Bernoulli random variables with unknown parameters $q_{s, k}$ and $p_{s, k}$ respectively. $X_{s,k}=1$ indicates that dose $k$ is effective for subgroup $s$ and $Y_{s,k}=1$ indicates that dose $k$ is unsafe for subgroup $s$.
We define a dose to be unsafe for subgroup $s$ if the expected toxicity ($p_{s, k}$) for subgroup $s$ exceeds a pre-defined toxicity threshold $\zeta$, which is referred to as the MTD (maximum tolerated dose) threshold in the clinical trial literature.
We also define a minimum efficacy threshold, $\theta$, below which we do not wish to administrate treatment because the efficacy is too low.

For $p_{s,k}$, we introduce the logistic dose-toxicity model $
p_{s,k}(a)=\left( \frac{\tanh u_k+1}{2} \right)^{a},
$ as in \cite{o1990continual},
where $u_k$ is an actual dose level of dose $k$ and $a$ is a parameter used for all doses.
Then, due to the monotonicity of the dose-toxicity model, we can partition doses into safe/unsafe doses as $\{p_{s,1},...,p_{s,U_s}\}$ and $\{p_{s,U_s+1},...,p_{s,K}\}$,
where $U_s=\max\{k\in\mathcal{K}:p_{s,k}<\zeta\}$.

At the beginning of each round, $t$, of the clinical trial, a patient arrives.
We denote the subgroup of the patient in round $t$ by $H(t)\in\mathcal{S}$.
The arrival rate of the patients in each subgroup $s$ is given by $\xi_s$.
Then, the probability of the patient in round $t$ being from subgroup $s$ is given by $\pi_s = \mathbb{P}(H(t)=s)=\frac{\xi_s}{\sum_{s\in\mathcal{S}}\xi_s}$.
For the observed subgroup $H(t)$, an agent of the clinical trial chooses a dose $k\in\{0\}\cup\mathcal{K}$,
 where ``0'' represents a ``no-dose'' action corresponding to the agent skipping the new patient.
We denote the dose allocated in round $t$ by $I(t)$.
Let $X_t$, $Y_t$, and $Z_t$ be the efficacy, toxicity, and cost in round $t$, respectively.
When $I(t)\geq 1$, we have  $X_t=X_{H(t),I(t)}$, $Y_t=Y_{H(t),I(t)}$, and $Z_t=1$ (we will generalize this to heterogeneous costs across subgroups in Section \ref{sec:heterogeneous_costs}).
When $I(t)=0$, we have no efficacy or toxicity and the cost $Z_t = 0$.
The clinical trial ends when the budget is exhausted or at the end of time-horizon $T$.

\section{C3T WITH LIMITED BUDGET}
\label{sec:algorithm}
\subsection{Problem Formulation}
\label{sec:prob_formulation}
Let $\Pi$ be a bandit algorithm that maps historical observations, $\{(X_\tau,Y_\tau,H(\tau),I(\tau)\}_{\tau=1}^{t-1}$, and the current subgroup, $H(t)$, to a dose $I(t)\in\{0\}\cup\mathcal{K}$.
We define the average toxicity of subgroup $s$ to be 
$S_{\Pi,s}(T,B)\!=\!\mathbb{E}\left[\frac{\sum_{t=1}^{T} \mathbb{I}\{H(t)=s\}Y_t}{\sum_{t=1}^{T} \mathbb{I}\{H(t)=s\}\mathbb{I}\{I(t)\neq0\}}\right]$.
We define the total expected cumulative efficacy as
$E_\Pi(T,B)=\mathbb{E}\left[ \sum_{t=1}^{T} X_t \right]$.
To make a recommendation, the MTD threshold and minimum efficacy threshold should be considered.
We define the set of candidate doses for subgroup $s$ by
$
\mathcal{K}_s = \{k\in\mathcal{K}:q_{s,k} \geq \theta, p_{s,k} \leq \zeta\}.
$
We define the optimal dose-to-recommend for subgroup $s$ by
\vspace{-0.1in}
\begin{equation}
\label{eqn:best_dose}
{k_s^*}=\left\lbrace
\begin{array}{ll}
\argmax_{k\in\mathcal{K}_s} q_{s,k},&\textrm{if }\mathcal{K}_s\neq \emptyset, \\
0, &\textrm{otherwise}.
\end{array}\right.
\end{equation}
We define the recommended dose to subgroup $s$ at the end of the clinical trial as $\hat{k}_s^*(T,B)$.
Then, the dose recommendation error can be written as $D_\Pi(T,B)=\sum_{s\in\mathcal{S}}\mathbb{E}\left[\mathbb{I}[\hat{k}_s^*(T,B)\!\neq\! {k_s^*}]\right]$.
Our C3T problem of minimizing the dose recommendation error while satisfying the budget and safety constraints is formally presented as
\vspace{-0.1in}
\begin{IEEEeqnarray}{cl}
\label{eqn:problem}
\minimize ~& D_\Pi(T,B) \nonumber \\
\subjecto ~&  \mathbb{P}\left[S_{\Pi,s}(T,B)\leq\zeta\right] \geq 1-\delta_s,~\forall s\in\mathcal{S}  \\
& \textstyle\sum_{t=1}^{T} Z_t\leq B. \nonumber
\end{IEEEeqnarray}
where $\delta_s$ is the maximum probability with which the toxicity for subgroup $s$ can exceed the MTD threshold.
%The budget constraint is a hard constraint that should not be violated under any realization.
%The safety constraint makes the average toxicity of each subgroup keep under the MTD threshold $\zeta$.
%

\subsection{Addressing the Budget Constraint}
In clinical trials with a single subgroup, maximizing the cumulative efficacy with the well-known principle of optimism in the face of uncertainty
is also effective to recommend dose \citep{aziz2019multi}.
Such objective makes the agent more frequently choose the doses that are more likely  to be the optimal dose.
The dose recommendation accuracy is improved since the accuracy for any dose depends primarily on the number of patients that were given that dose. Since patients are more likely to receive the optimal dose, the accuracy of the optimal dose is better.
However, in the case of multiple subgroups and limited budget, maximizing  the cumulative efficacy is no longer effective at recommending doses across {\em all} subgroups due to the varying efficacies and arrival rates.
For example, if one subgroup has much higher efficacy than others, the optimal strategy for the agent would involve skipping all patients that arrive from subgroups other than the one with the highest efficacy.
Little will therefore be learned about the response of other subgroups to the drug, resulting in a poor learning performance.
Hence, in each round, the agent should decide to skip the new patient or not based {\em also} on the estimation accuracy for the patient's subgroup, remaining budget, and remaining time.

We consider an oracle problem with budget $B$  and time-horizon $T$ to highlight how to address the budget constraint. 
%To explain how we address the budget constraint, we first consider a simple problem with budget, $B$, and time-horizon, $T$.
We assume $S$ subgroups with different arrival rates $\pi_s$, $s = 1, ..., S$.
Each subgroup $s$ has an arbitrary expected reward $d_s^*$.
For convenience, we assume $d_1^* > d_2^* > ... > d_S^*$.
In general, the expected rewards $d_s^*$ are unknown to the agent.
However, we begin by considering an oracle agent that knows $d_s^*$ and has only two actions: skip or accept the patient.
In each round $t$, the agent will determine whether to skip the new patient or not based on the remaining rounds $\tau=T-t+1$ and the remaining budget $b_\tau$.

Although the oracle problem seems simple, it is known to be computationally intractable \citep{wu2015algorithms}.
We thus approximate this problem as a linear program (LP).
We first relax the problem by substituting the hard budget constraint with an average budget constraint using $\rho=\frac{B}{T}$.
We define the probability that the agent does not skip a patient in subgroup $s$ by $\psi_s$,
and write $\bpsi=\{\psi_1,\psi_2,...\psi_S\}$.
We now formulate the following LP problem:
\begin{equation}
\label{eqn:lp_problem}
\maximize_{\bpsi} \sum_{s\in\mathcal{S}}\psi_s\pi_s d_s^* ~~
\subjecto \sum_{s\in\mathcal{S}}\psi_s\pi_s\leq \rho.
\end{equation}
The optimal solution of this problem can be easily derived:
\begin{equation}
\label{eqn:lp_solution}
\psi_s(\rho)=\left\lbrace\begin{array}{ll}
1,& \textrm{if } 1\leq s\leq \tilde{s}(\rho), \\
\frac{\rho-\eta_{\tilde{s}(\rho)}}{\pi_{\tilde{s}(\rho)+1}}, & \textrm{if } s=\tilde{s}(\rho)+1, \\
0,& \textrm{if } s>\tilde{s}(\rho)+1,
\end{array}\right.
\end{equation}
where $\eta_s=\sum_{s'=1}^s\pi_{s'}$ (with the convention that $\eta_0=0$) and $\tilde{s}(\rho)=\max\{s\in\mathcal{S}\cup\{0\}:\eta_s\leq\rho\}$.
%By using the solution, the optimal value of the LP problem can be obtained as
%\[
%v(\rho)=\sum_{s=1}^{\tilde{s}(\rho)}\pi_s d_s^*+\psi_{\tilde{s}(\rho)+1}(\rho)\pi_{\tilde{s}(\rho)+1}d_{\tilde{s}(\rho)+1}^*.
%\]
%The optimal value $v(\rho)$ can be considered the maximum expected reward in a single round with average budget $\rho$.

%This approximation therefore provides us with a policy that addresses a budget constraint while considering arbitrary expected rewards in \eqref{eqn:lp_solution}.
We can adapt this policy to our more general problem by deciding to skip a patient or not by replacing $d_s^*$ and $\rho$ in \eqref{eqn:lp_problem} with the estimated recommendation accuracy or estimated efficacy and the remaining average budget $\rho_\tau=\frac{b_\tau}{\tau}$, respectively.

\subsection{C3T-Budget}
In Algorithm \ref{alg:algorithm}, we propose a bandit solution for the budget-limited C3T problem in \eqref{eqn:problem} which we call C3T-Budget.
As noted above, when there is only 1 subgroup, maximizing the cumulative efficacy using the optimism principle is effective for recommending dose.
Thus, in the algorithm, the dose for each subgroup is chosen according to the upper confidence bound (UCB) principle using the estimated efficacy and toxicity of the subgroup.
On the other hand, the decision to skip a patient is based on how convincing the estimated efficacy of the chosen dose for each subgroup is.

\begin{algorithm}[ht]
\footnotesize
\caption{C3T-Budget}
\label{alg:algorithm}
\begin{algorithmic}[1]
\State \textbf{Input:} Time-horizon $T$, budget $B$, subgroup arrival distributions $\pi_s$'s, $\phi$ for credible intervals
\State \textbf{Initialize:} $\tau=T$, $b=B$, $t=1$, $\alpha^{\textrm{Beta}}_{s,k}(0)=1$, $\beta^{\textrm{Beta}}_{s,k}(0)=1,\forall s\in\mathcal{S},\forall k\in\mathcal{K}$.
\While{$t\leq T$}
\State $\hat{a}_s(t)\leftarrow \frac{\sum_{k=1}^{K} \hat{a}_{s,k}(t-1)N_{s,k}(t-1)}{N_{s}(t-1)},\forall s\in\mathcal{S}$
\State $\mathcal{K}_s(t)\!=\!\{k\!\in\!\mathcal{K}\!:\hat{q}_{s,k}(t)\!\geq\!\theta, p_{s,k}(\hat{a}_s(t)\!+\!\alpha_s(t))\leq\zeta\},\forall s\!\in\!\mathcal{S}$
\If{$b > 0$}
\If{$N_{H(t)}(t) \leq K$}
\State{Sample each dose once $I(t)=N_{H(t)}(t)$}
\Else
\State $k_s^*(t)\leftarrow\argmax_{k\in\mathcal{K}_{s}(t)} \hat{q}_{s,k}(t),\forall s\in\mathcal{S}$
\parState{Calculate $B_s^*(t)$ as in \eqref{eqn:improvement_credible}}
\parState{Obtain $\hat{\bpsi}\left(b/\tau\right)$'s by solving the LP problem in \eqref{eqn:lp_problem} with ordered $B_s^*(t)$'s}
\parState{Allocate dose 

$
I(t)\!=\!\left\lbrace\begin{array}{ll}
\!\!k_{H(t)}^*(t),& \!\!\!\textrm{with probability }\hat{\psi}_{H(t)}(b/\tau), \\
\!\!0,&  \!\!\!\textrm{otherwise}.
\end{array}\right.
$}
\EndIf
\EndIf
\State Observe the efficacy $X_t$ and toxicity $Y_t$
\parState{Update $\tau$, $b$, $N_s(t)$, $N_{s,k}(t)$, $\bar{q}_{s,k}(t)$, $\hat{q}_{s,k}(t)$, $\bar{p}_{s,k}(t)$, $\alpha^{\textrm{Beta}}_{s,k}(t)$, $\beta^{\textrm{Beta}}_{s,k}(t)$}
\parState{$\hat{a}_{s,k}(t)\!\leftarrow\!\argmin_a|p_{s,I(t)}(a)\!-\!\bar{p}_{s,I(t)}(t)|,\forall s\in\mathcal{S},\forall k\in\mathcal{K}$}
\State $t=t+1$
\EndWhile
\parState{\textbf{Output:}  Recommended dose

$
{\hat{k}_s^*}=\left\lbrace
\begin{array}{ll}
\argmax_{k\in\mathcal{K}'_s(T)} \bar{q}_{s,k},&\textrm{if }\mathcal{K}'_s(T)\neq \emptyset, \\
0, &\textrm{otherwise}
\end{array}\right., \forall s\in\mathcal{S}
$}
\end{algorithmic}
\end{algorithm}
\vspace{-1em}

We define the empirical efficacy estimation of dose $k$ for subgroup $s$ at round $t$ by $\bar{q}_{s,k}(t)=\frac{\sum_{\tau=1}^t \mathbb{I}\{H(\tau)=s,I(\tau)=k\}X_\tau}{N_{s,k}(t)}$, where $N_{s,k}(t)$ is the number of times that dose $k$ has been allocated to subgroup $s$ up to round $t$ (i.e., $N_{s,k}(t)=\sum_{\tau=1}^{t} \mathbb{I}\{H(\tau)=s,I(\tau)=k\}$).
Similarly, we define the empirical toxicity estimation of dose $k$ for subgroup $s$ at round $t$ by $\bar{p}_{s,k}(t)=\frac{\sum_{\tau=1}^t \mathbb{I}\{H(\tau)=s,I(\tau)=k\}Y_\tau}{N_{s,k}(t)}$.
Then, the UCB of $q_{s,k}$ at round $t$ is defined as $\hat{q}_{s,k}(t)=\bar{q}_{s,k}+\sqrt{\frac{c\log N_s(t)}{N_{s,k}(t-1)}}$, where $N_s(t)$ is the number of times that subgroup $s$ has arrived up to round $t$ (i.e., $N_s(t)=\sum_{\tau=1}^{t} \mathbb{I}\{H(\tau)=s\}$).
The confidence interval $a_s$ for the dose-toxicity model of subgroup $s$ is given by
$\alpha_s(t)=CK\left(\frac{\log\frac{2K}{\delta_s}}{2N_s(t)}\right)^{\frac{\gamma}{2}}$, where $C$ and $\gamma$ are hyper-parameters of our model.
Details of hyper-parameter selection are given in the supplementary material.

In each round $t$, the algorithm constructs a set of candidate doses for each subgroup $s$, $\mathcal{K}_s$, by considering the estimated expected efficacy and toxicity of each dose for the subgroup and using the UCB principle.
Among the candidate doses, the algorithm selects the estimated optimal dose in round $t$, $k_s^*(t)$, which has the largest UCB of the expected efficacy for subgroup $s$, as $k_s^*(t)=\argmax_{k\in\mathcal{K}_s(t)}\hat{q}_{s,k}(t)$.
Then, the agent determines whether the patient in round $t$ is to be skipped or not by considering how convincing the estimation of the efficacy of $k_s^*$ is.
To do this, we use the credible interval of the estimation of $\bar{q}_{s,k}$. 

We adopt a Bayesian approach to estimate $\bar{q}_{s,k}$.
We first consider a uniform distribution (i.e., $\textrm{Beta}(1,1)$) as the prior for $q_{s,k}$.
Then, the maximum posterior estimation of $q_{s,k}$ becomes the empirical efficacy estimation $\bar{q}_{s,k}$.
This allows us to combine the Bayesian approach and the UCB principle without conflict.
We define the parameters of the prior distribution of $q_{s,k}$ in round $t$ as $\alpha^{\textrm{Beta}}_{s,k}(t)$ and $\beta^{\textrm{Beta}}_{s,k}(t)$.
Then, at the end of round $t$, we update the posterior distribution of $q_{H(t),I(t)}$ as $\alpha^{\textrm{Beta}}_{H(t),I(t)}(t)=\alpha^{\textrm{Beta}}_{H(t),I(t)}(t-1)+X(t)$ and $\beta^{\textrm{Beta}}_{H(t),I(t)}(t)=\beta^{\textrm{Beta}}_{H(t),I(t)}(t-1)+(1-X(t))$.
By using the posterior distribution of $q_{s,k}$, we can obtain a credible interval for a given probability $\phi$.
Let $f(\phi,\alpha^{\textrm{Beta}},\beta^{\textrm{Beta}})$ be a function that calculates the interval length to achieve the probability $\phi$.
We then define the expected improvement of the credible interval for an additional patient in subgroup $s$ with dose $k$ as
\vspace{-0.1in}
\begin{IEEEeqnarray}{l}
\label{eqn:improvement_credible}
B_{s,k}\!=\!\bar{q}_{s,k}(f(\phi,\alpha^{\textrm{Beta}}_{s,k},\beta^{\textrm{Beta}}_{s,k})\!-\!f(\phi,\alpha^{\textrm{Beta}}_{s,k}+1,\beta^{\textrm{Beta}}_{s,k})) \nonumber \\
~ +(1\!-\!\bar{q}_{s,k})(f(\phi,\alpha^{\textrm{Beta}}_{s,k},\beta^{\textrm{Beta}}_{s,k})\!-\!f(\phi,\alpha^{\textrm{Beta}}_{s,k},\beta^{\textrm{Beta}}_{s,k}\!+\!1)).\quad~
\end{IEEEeqnarray}
This value quantifies the improvement in estimation accuracy provided by including an additional patient from subgroup $s$ into the trial.
Thus, we can use it to determine which subgroups should be skipped to maximize the dose-recommendation accuracy.
Specifically, in round $t$, we rearrange the set $\{B_{s,k_s^*(t)}(t) : s = 1, ..., S\}$ in descending order.
The algorithm then solves the LP problem in \eqref{eqn:lp_problem}, substituting $d_s^*$ and $\rho$ with $B_{s,k_s^*(t)}(t)$ and $\rho_\tau$, respectively, and obtains a vector of probabilities $\hat{\bpsi}(\rho_\tau)$.
Finally, at the end of the clinical trial, the dose recommended to subgroup $s$ is given by\looseness=-1
\[
{\hat{k}_s^*}(T,B)=\left\lbrace
\begin{array}{ll}
\argmax_{k\in\mathcal{K}_s(T)} \bar{q}_{s,k},&\textrm{if }\mathcal{K}'_s(T)\neq \emptyset, \\
0, &\textrm{otherwise}.
\end{array}\right.
\]

\subsection{Extension: C3T-Budget-E}

In a clinical trial where  the budget is particularly limited, identifying even just one subgroup for which the drug is effective can be more important than understanding its effect on the whole population of interest. Trying to learn on the whole population might lead to the responses {\em all} subgroups being poorly estimated and the drug failing to progress at all.  Thus,  a different problem formulation is used when the budget is small, allowing the trial to focus on subgroups with high efficacies:
\begin{IEEEeqnarray}{cl}
\label{eqn:problem2}
\maximize ~& E_\Pi(T,B) \nonumber \\
\subjecto ~&  \mathbb{P}\left[S_{\Pi,s}(T,B)\leq\zeta\right] \geq 1-\delta_s,~\forall s\in\mathcal{S}  \\
& \textstyle\sum_{t=1}^{T} Z_t\leq B. \nonumber
\end{IEEEeqnarray}
where we have simply substituted the objective function $D_\Pi(T,B)$ in \eqref{eqn:problem} with $E_\Pi(T,B)$.
With this formulation, the agent tries to  accurately recommend doses for subgroups with high efficacies while also achieving high total cumulative efficacy.%achieve high efficacies (rather than low dose recommendation error), which results in focusing on subgroups with high efficacies.

We now propose a bandit algorithm for the budget-limited C3T problem in  \eqref{eqn:problem2} to maximize the total cumulative efficacy called C3T-Budget-E.
This algorithm focuses more on the subgroups that have high efficacies, which might result in inaccurate dose recommendation for subgroups with low efficacies.
%However, as discussed in Section \ref{sec:prob_formulation},
%when the budget is particularly limited, identifying even just one subgroup for which the drug is effective can be more important than understanding its effect on the whole population of interest. Trying to learn on the whole population might lead to the responses {\em all} subgroups being poorly estimated and the drug failing to progress at all. To achieve this goal, we focus on accurately recommending doses for subgroups with high efficacies while also achieving high total cumulative efficacy.

C3T-Budget-E is a modified version of C3T-Budget. The key difference is that the decision to skip a patient is based on the expected efficacy of each subgroup rather than the expected improvement of the credible interval.
Hence, in C3T-Budget-E, the parameters for the Bayesian posterior distribution are not used.
Specifically, the algorithm first finds the largest UCB of the estimated expected efficacy for subgroup $s$ in round $t$ as $\hat{q}_s^*(t)=\max_{k\in\mathcal{K}_s(t)}\hat{q}_{s,k}(t)$.
Then, it re-orders $\{\hat{q}_s^*(t): s=1, ..., S\}$ and obtains the vector of probabilities $\hat{\bpsi}(\rho(t))$ by solving the LP problem in \eqref{eqn:lp_problem} with the ordered $\hat{q}_s^*(t)$'s and $\rho_\tau$.
A detailed description of C3T-Budget-E can be found in the supplementary material.

\subsection{Theoretical Analysis}
\label{sec:analysis}
We now analyze the theoretical performance of the proposed algorithms, which we refer to collectively as C3T-Budgets (i.e., C3T-Budget and C3T-Budget-E).
All proofs can be found in the supplementary material.
We begin by bounding the safety constraint violation.
\begin{theorem}
\label{thm:safety_violation_bound}
For any given $T$, the average toxicity of subgroup $s$ observed from C3T-Budgets satisfies
\[
\mathbb{P}\left[\frac{\sum_{t=1}^{N_s(T)}p_{s,I(N_s^{-1}(t))}}{N_s(T)}-\zeta\leq C\epsilon^{\gamma}\right]\geq 1-\delta_s,
\]
for an arbitrary $\epsilon>0$, where $C$ and $\gamma$ are hyper-parameters (described in the supplementary material).
\end{theorem}
C3T-Budget and C3T-Budget-E therefore satisfy the safety constraints given in \eqref{eqn:problem} and \eqref{eqn:problem2}, respectively.

The following theorem bounds the recommended dose error of C3T-Budgets.
\begin{theorem}
\label{thm:recommendation}
The probability that C3T-Budgets recommends an incorrect dose for subgroup $s$ is bounded according to
\begin{IEEEeqnarray}{l}
\mathbb{P}\left[\hat{k}_s^*\neq k_s^*\right]\leq M_{R1}e^{-M_{R2} N_s(T)} \nonumber
\end{IEEEeqnarray}
where $M_{R1}$ and $M_{R2}$ are non-negative constants (provided in the supplementary material).
\end{theorem}
Theorem \ref{thm:recommendation} indicates that the accuracy of the recommended dose depends strongly on the number of samples accepted in each subgroup.
%The behavior of $N_s(T)$ within each algorithm is difficult to analyze, but we can intuitively anticipate their behavior.
%In C3T-Budget, the subgroups are recruited to improve the estimation credibility on each subgroup.
%Thus, we would expect that the number of samples are determined in such a way as to balance recommended dose errors across all subgroups.
%On the other hand, in C3T-Budget-E, the agent will want to recruit patients from subgroups with high efficacies.
%Thus, we would expect that the recommended dose errors of such subgroups are lower, while the recommended dose errors for subgroups with low efficacy will be higher.
The worst-case regret bound for the total efficacy of C3T-Budget-E is provided in the supplementary material.
We also corroborate these theoretical behaviors of C3T-Budgets with empirical results in the experiments section.

\section{DISCUSSION}
\subsection{Heterogeneous Costs for Subgroups}
\label{sec:heterogeneous_costs}
In practice, the cost of recruiting a patient into a clinical trial may be different across subgroups.
To address this, an extension of C3T-Budgets is needed in which heterogeneous costs for different subgroups can be accounted for.
Let $c_{s}$ be a non-negative cost associated with subgroup $s$ that occurs when any dose is allocated to a patient in subgroup $s$.
We can then reformulate the LP problem in \eqref{eqn:lp_problem} as
\[
\maximize_{\bpsi} \sum_{s\in\mathcal{S}}\psi_s\pi_s' \frac{d_s^*}{c_s} ~~
\subjecto \sum_{s\in\mathcal{S}}\psi_s\pi_s'\leq \rho,
\]
where $\pi_s'=\pi_s c_s$.
Then we apply C3T-Budget as in Algorithm \ref{alg:algorithm} by substituting $\pi_s'$ and $B_s^*(t)/c_s$ ($q^*_s(t)/c_s$ for C3T-Budget-E) into line 12, whenever appropriate.

\subsection{Sequential Patient Recruitment}
In clinical trials, a ``complete" list of candidate patients is often given in advance, and then patients are sequentially selected from among the candidate patients.
For this sequential patient recruitment, C3T-Budgets can be applied.
Let $\tilde{\pi}_s$ be proportion of candidate patients that belong to subgroup $s$.
Based on $\tilde{\pi}_s$, a patient arrival sequence can be virtually generated and the problem reduces to the same problem as originally discussed, so that C3T-Budgets can be applied to the virtually generated arrival sequence.

\section{EXPERIMENTS}
In this section, we demonstrate the performance of our algorithm through a series of experimental results.\footnote{The implementation of C3T-Budgets is available at: \texttt{https://bitbucket.org/mvdschaar/mlforhealthlabpub/\\ src/master/alg/c3t\_budgets}}
We consider a dose-finding clinical trial with 3 subgroups (i.e. $S=3$) and 6 doses (i.e. $K=6$).
In the clinical trial, 400 patients can be recruited in total (i.e., $B=400$) and a maximum of 1200 rounds can be performed (i.e. $T=1200$). The arrival distributions for each of the 3 subgroups is given by $\pi_1 = \frac{5}{12}$, $\pi_2=\frac{4}{12}$, and $\pi_3=\frac{3}{12}$.
The MTD threshold is set to be 0.35 (i.e. $\zeta=0.35$) and the minimum efficacy threshold is set to be 0.2 (i.e. $\theta=0.2$).
The expected efficacy, $q_{s,k}$, and expected toxicity, $p_{s,k}$, of each dose for each subgroup are provided in Table \ref{table:tox_eff_prob}.
We highlight the optimal dose for each subgroup in bold.
Note that there is no optimal dose for subgroup 1 as in \eqref{eqn:best_dose} since their expected efficacy is below the minimum efficacy threshold (and so the optimal action is to not treat them at all).
All experimental settings are similar to \cite{riviere2018phase,aziz2019multi} and we use them unless otherwise mentioned.
All experiments are repeated 500 times and the results are averaged.

\begin{table}[ht]
\caption{Settings of Synthetic Model} \label{table:tox_eff_prob}
\footnotesize
\begin{center}
\begin{tabular}{c|c|cccccc}
\hline
$S$ & \textbf{Type}  &\multicolumn{6}{c}{\textbf{Characteristics}} \\
\hline
\multirow{2}{*}{1} & $q_{1,k}$'s & 0.01 & 0.02 & 0.05 & 0.1 & 0.1 & 0.1 \\
& $p_{1,k}$'s & 0.01 & 0.01 &0.05 & 0.15 &0.2 &0.45 \\
\hline
\multirow{2}{*}{2} & $q_{2,k}$'s & 0.1 & 0.2 & 0.3 & \textbf{0.5} & 0.6 & 0.65 \\
& $p_{2,k}$'s & 0.01 & 0.05 & 0.15 & \textbf{0.2} & 0.45 & 0.6 \\
\hline
\multirow{2}{*}{3} & $q_{3,k}$'s & 0.2 & 0.5 & 0.6 & \textbf{0.8} & 0.84 & 0.85 \\
& $p_{3,k}$'s &  0.01 & 0.05 & 0.15 & \textbf{0.2} & 0.45 & 0.6 \\
\hline
\end{tabular}
\end{center}
\end{table}
           
For benchmarks, we modify the following baseline algorithms to include contexts as in \cite{zhou2015survey} and implement them for contextual clinical trials: 3+3 design \citep{storer1989design}, Contextual UCB \citep{auer2002finite,varatharajah2018contextual}, Contextual KL-UCB \citep{garivier2011kl,varatharajah2018contextual}, and Contextual independent Thompson sampling (TS) \citep{aziz2019multi}.
The details of the algorithms are provided in the supplementary material.

\subsection{Recommended Dose Error Rates}
In Table \ref{table:dose_errors}, we report the recommended dose error rate of each algorithm both as an average across all subgroups (Total) and individually for each subgroup (SG1, SG2, SG3).
From the results, we can see that C3T-Budget achieves the lowest total error rate, with a lower error within each subgroup than all benchmarks except: C-UCB and C-Indep-TS whose error is lower in SG1 due to the fact that SG1 has the highest arrival rate and C-UCB and C-Indep-TS recruits all patients, C3T-Budget instead makes better use of the budget to {\em balance} the learning across all subgroups; and C3T-Budget-E, which notably only achieves lower error on SG3, the subgroup with the highest efficacy. This is exactly as we would expect, since the goal of C3T-Budget-E is to focus on the highest efficacy subgroups, thus resulting in better estimations of the optimal dose for those subgroups.
The superior performance of C3T-Budget is due to its capability to effectively recruit patients from subgroups that will best improve recommended dose errors while remaining within the budget constraint.
This will be investigated further in Section \ref{sec:sim:ratio_subgroup}.
%During the trials, C3T-Budget skips the subgroups whose recommended doses are convincing enough considering the limited budget and time.

\begin{table}[ht]
\vspace{-0.15in}
\caption{Recommended Dose Error Rates} \label{table:dose_errors}
\footnotesize
\begin{center}
\begin{tabular}{c|c|c|c|c}
\hline
\multirow{2}{*}{\textbf{Algorithm}} & \multicolumn{4}{c}{\textbf{Errors}} \\
\cline{2-5}
& SG 1 & SG 2 & SG 3 & Total \\
\hline
C3T-Budget	& 0.050 & \textbf{0.056} & 0.036 & \textbf{0.047} \\
C3T-Budget-E& 0.204 & 0.138 & \textbf{0.020} & 0.121 \\
C-UCB		& \textbf{0.012} & 0.414 & 0.316 & 0.247 \\
C-KL-UCB    & 0.204 & 0.440 & 0.368 & 0.337 \\
C-Indep-TS	& \textbf{0.012} & 0.354 & 0.372 & 0.246 \\
C-3+3		& 0.874 & 0.746 & 0.780 & 0.800 \\
\hline
\end{tabular}
\end{center}
\end{table}

\subsection{Safe Dose Estimation Error Rates}
We define the type-I and type-II errors for dose safety as follows: Type-I error occurs when a safe dose is estimated to be unsafe and type-II error occurs when an unsafe dose is estimated to be safe.
We report these errors for C3T-Budgets and each of the benchrmark algorithms in Table \ref{table:safe_errors}.
We see that C3T-Budget achieves the lowest total error rates.
In C3T-Budget, subgroups are recruited to maximize the confidence of our estimations, which in turn leads to confident estimates of safe doses.

\begin{table}[ht]
\vspace{-0.15in}
\caption{Safe Dose Estimation Error Rates} \label{table:safe_errors}
\footnotesize
\begin{center}
\begin{tabular}{c|c|c|c}
\hline
\multirow{2}{*}{\textbf{Algorithm}} & \multicolumn{3}{c}{\textbf{Errors}} \\
\cline{2-4}
& Type-I & Type-II & Total \\
\hline
C3T-Budget	& 0.0226 & \textbf{0.0198} & \textbf{0.0212} \\
C3T-Budget-E& \textbf{0.0200} & 0.0304 & 0.0252 \\
C-UCB		& 0.0248 & 0.0289 & 0.0268 \\
C-KL-UCB    & 0.0371 & 0.0437 & 0.0404 \\
C-Indep-TS	& 0.1006 & 0.0366 & 0.0686 \\
C-3+3		& 0.0400 & 0.1081 & 0.0741 \\
\hline
\end{tabular}
\end{center}
\end{table}

\subsection{Efficacy and Toxicity Per Patient}
We now report the efficacy and toxicity per patient for each algorithm in Table \ref{table:efficacy}.
From the results, we can see that both C3T-Budgets, but particularly C3T-Budget-E, achieve a much higher efficacy than other algorithms.
In addition, C3T-Budgets achieve a lower toxicity than other algorithms except for the 3+3 design.
Note, however, that the 3+3 design is designed specifically to avoid safety violations - it ends when a certain number of toxicities have occurred.
Hence, it naturally has a low toxicity compared to the bandit algorithms, but it achieves a much lower efficacy as well as shown in Table \ref{table:efficacy} and inaccurately recommends doses as shown in Table \ref{table:dose_errors}.

\begin{table}[ht]
\vspace{-0.17in}
\caption{Efficacy and Toxicity Per Patient} \label{table:efficacy}
\footnotesize
\begin{center}
\begin{tabular}{c|c|c}
\hline
\textbf{Algorithm} & \textbf{Efficacy} & \textbf{Toxicity} \\
\hline
C3T-Budget & 0.4975 & 0.1881 \\
C3T-Budget-E & \textbf{0.5791} & 0.1911 \\
C-UCB & 0.4154 & 0.3235 \\
C-KL-UCB & 0.3823 & 0.2621 \\
C-Indep-TS & 0.4101 & 0.3157 \\
C-3+3 & 0.3081 & \textbf{0.1537} \\
\hline
\end{tabular}
\end{center}
\end{table}

\subsection{Recruited Patients for Each Subgroup and Efficacy Estimation Errors}
\label{sec:sim:ratio_subgroup}
To understand the recruitment behaviors of C3T-Budgets and how they relate to efficacy estimation error, Fig. \ref{fig:ratio} shows the number of patients recruited from each subgroup and the mean squared error (MSE) of the efficacy estimate of the optimal dose in each subgroup (for subgroup 1, we provide the averaged MSE of all safe doses since there is no optimal dose) over time.
Note that the recruitment numbers of the other bandit algorithms are equal to those with $\pi_s$ since they do not address the budget constraint.

\begin{figure}[ht]
\vspace{-0.1in}
\centering
\includegraphics[width=1\linewidth]{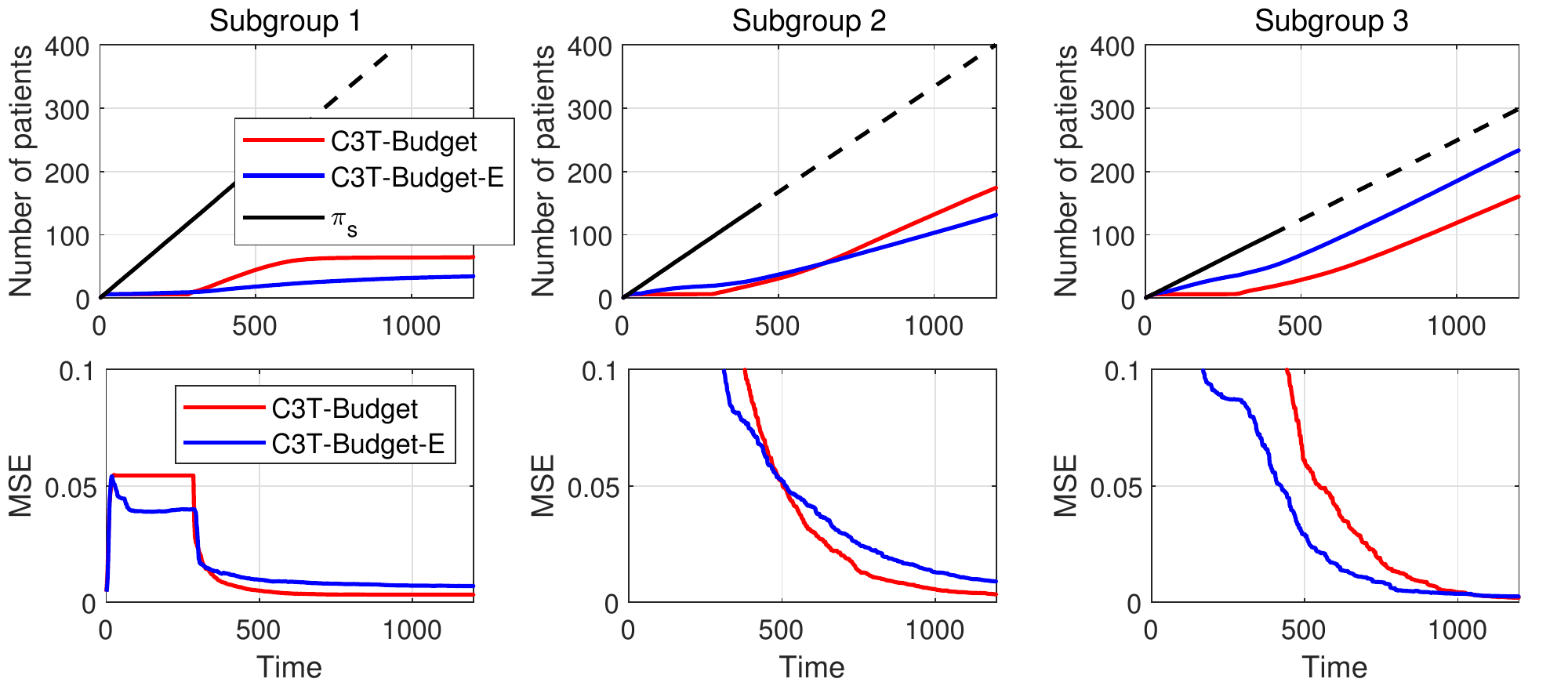}
\caption{Number of recruited patients and the MSE of the efficacy estimation of each subgroup.\label{fig:ratio}}
\end{figure}

From the figure, we can see that both C3T-Budgets recruit patients independently of the arrival rate of each subgroup.
In C3T-Budget, the patients are recruited to achieve low enough estimation errors for all subgroups, demonstrated by the fact that recruitment for subgroup 1 is stopped after the estimation errors of the doses for subgroup 1 become low enough.
C3T-Budget then starts to recruit more patients from subgroups 2 and 3.
Subgroup 2 ends up having the most patients recruited for it, due to the fact that the expected efficacy (0.5) of its optimal dose is more difficult to estimate than those of subgroups 1 and 3 (0.1 and 0.8, respectively).
On the other hand, C3T-Budget-E recruits more patients from subgroup 3 due to it having the highest expected efficacy.
C3T-Budget-E recruits only enough patients from subgroup 1 to allow for reasonable belief that it has the lowest expected efficacy.
More patients are recruited from subgroup 2 than from subgroup 1, due to the fact that it requires more samples to be sure that the efficacy is lower than that of subgroup 3, but once the algorithm has confidence that the efficacy is lower, recruitment slows, in favour of recruiting from subgroup 3, thus leading to higher estimation error for subgroup 2.
We see from these figures that the recruitment behavior of C3T-Budget and C3T-Budget-E are as intended.

\subsection{Impact of Budget}
In this experiment, we report the recommended dose error rates and efficacy per patient of each algorithms while we vary the budget, $B$, in Fig. \ref{fig:var_B}.
Since the ratio $B/T$ affects C3T-Budgets, we set the time-horizon $T$ such that the ratio $B/T$ is constant and $1/3$.
The results of the 3+3 design are not provided here since its performance does not depend on the total number of patients to dose.

\begin{figure}[ht]
\vspace{-0.17in}
\centering
\subfloat[Total recommended dose error rate.]{\includegraphics[width=0.48\linewidth]{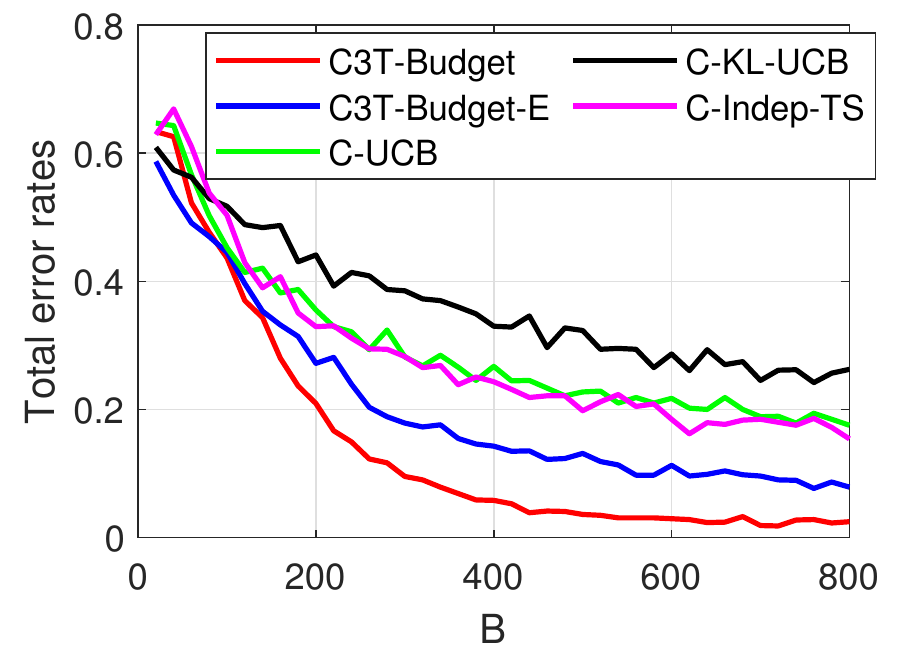}
\label{fig:dose_err_var_B}}
\hfil
\subfloat[Efficacy per patient.]{\includegraphics[width=0.48\linewidth]{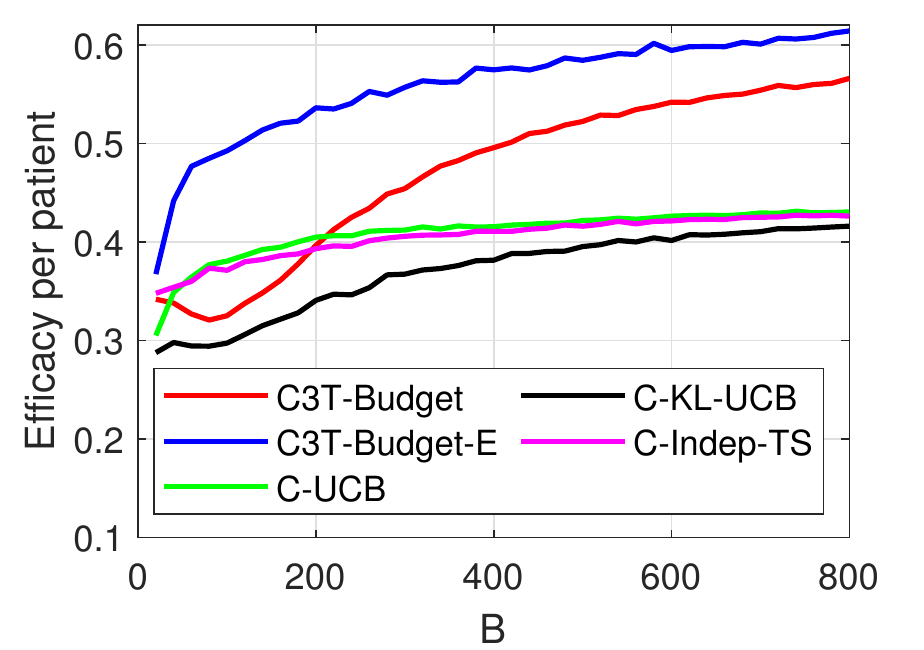}
\label{fig:eff_var_B}}
\caption{Total recommended dose error rate and efficacy per patient varying the budget.\label{fig:var_B}}
\end{figure}

From Figs. \ref{fig:dose_err_var_B} and \ref{fig:eff_var_B}, we can see that the recommended dose error rates of all algorithms decrease and the efficacy per patient of all algorithms increases, as $B$ increases.
However, we see that for all $B>80$, C3T-Budget achieves the lowest error, while for $B<80$ (i.e. when the budget is particularly limited), we see that C3T-Budget-E achieves the lowest error. For all $B$, we see that C3T-Budget-E achieves the best efficacy.
Notably, as $B$ increases, the error of C3T-Budget drastically decreases and achieves near-zero error when $B\geq500$.
These results clearly demonstrate that C3T-Budgets do indeed effectively utilize the budget constraints unlike the other algorithms.

\subsection{Impact of Time Horizon}
We show the impact of varying the time-horizon $T$ on the performance of our algorithm.
If the time-horizon $T$ is set to be the budget $B$, then C3T-Budgets do not skip any patients.
Thus, C3T-Budget and C3T-Budget-E become the same algorithm.
Note that all the other algorithms do not depend on the time-horizon $T$.

\begin{figure}[ht]
\vspace{-0.17in}
\centering
\subfloat[Total recommended dose error rate.]{\includegraphics[width=0.48\linewidth]{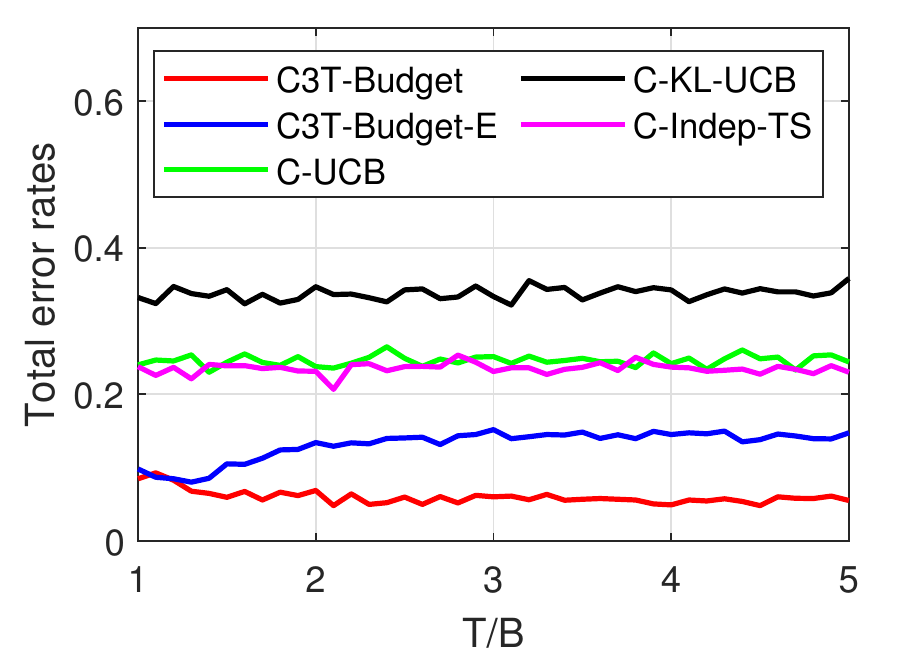}
\label{fig:dose_err_var_TB}}
\hfil
\subfloat[Efficacy per patient.]{\includegraphics[width=0.48\linewidth]{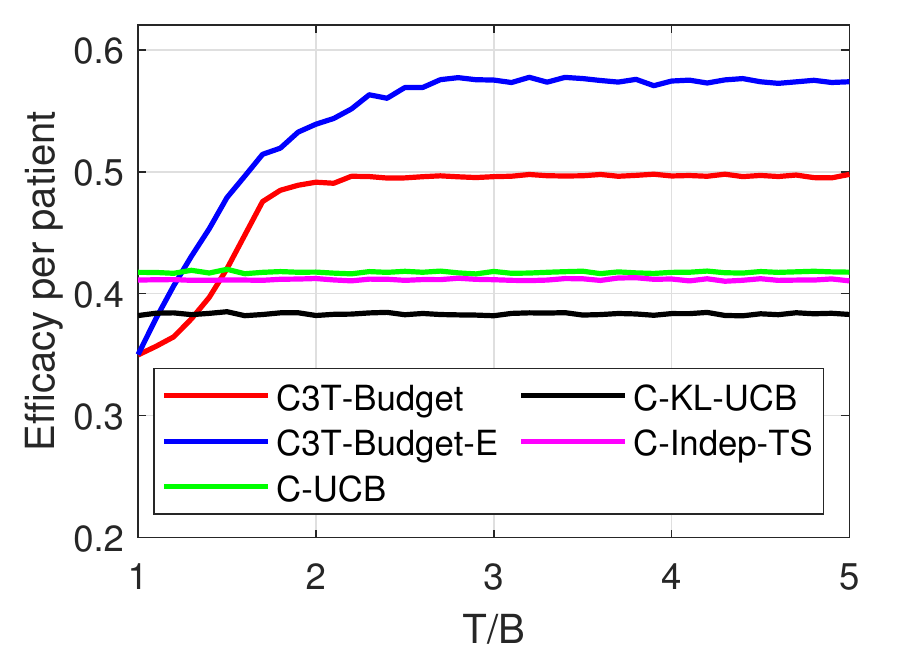}
\label{fig:eff_var_TB}}
\caption{Total recommended dose error rate and efficacy per patient varying the time-horizon.\label{fig:var_TB}}
\end{figure}

In Fig. \ref{fig:var_TB}, we show the recommended dose error rates and efficacy per patient of the algorithms when varying $T/B$ from 1 to 5.
From Fig. \ref{fig:dose_err_var_TB}, we can see that the recommended dose error rate of C3T-Budget decreases while that of C3T-Budget-E increases as $T/B$ increases.
%, and then, the error rates of both C3T-Bugets are saturated.
As $T/B$ increases, C3T-Budgets can recruit subgroups more flexibly, since they have more time in which to ``wait" for patients from a more desirable subgroup (they can afford to skip more patients).
Thus, the error rate of C3T-Budget decreases since it recruits patients that will best reduce the error.
On the other hand, the increased recruitment policy allows C3T-Budget-E to focus its recruitment more on subgroup 3 (so that it achieves high cumulative efficacy), but this in turn results in a higher error rate for subgroups 1 and 2, thus resulting in a higher total error.
From Fig. \ref{fig:eff_var_TB}, we can see that as $T/B$ increases, the efficacies of both C3T-Budgets increase thanks to the increased recruitment flexibility.
From the results, we can see that both error rate and efficacy are saturated.
If $T/B$ is large enough, even very small arrival rates for a given group can be accommodated. The limiting factor then becomes the budget, $B$, and increasing $T/B$ will have no (or very little) further effect.

\section{CONCLUSION}
In this paper, we have studied dose-finding clinical trials with heterogeneous groups and a limited budget.
%In such a trial, the limited budget needs to be effectively utilized to balance the trade-off between learning and treatment.
We proposed C3T-Budget in which, for each group, the dose with the highest estimated efficacy is chosen, but the group may be skipped if the algorithm is sufficiently confident in its current estimation of the efficacy.
In addition, we proposed an extension, C3T-Budget-E, which is particularly applicable in very low-budget settings, where accurately estimating efficacy for the most promising subgroups is more important than accurate estimation for all groups.
These designs can effectively utilize the limited budget to balance the trade-off between learning and treatment.
We demonstrated that our proposed algorithms outperform state-of-the-art algorithms through a series of experiments.

\bibliography{mybib}

\begin{thebibliography}{}

\bibitem[Aboutalebi et~al., 2019]{aboutalebi2019learning}
Aboutalebi, H., Precup, D., and Schuster, T. (2019).
\newblock Learning modular safe policies in the bandit setting with application
  to adaptive clinical trials.
\newblock {\em arXiv preprint arXiv:1903.01026}.

\bibitem[Atan et~al., 2019]{atan2019sequential}
Atan, O., Zame, W.~R., and Schaar, M. (2019).
\newblock Sequential patient recruitment and allocation for adaptive clinical
  trials.
\newblock In {\em The 22nd International Conference on Artificial Intelligence
  and Statistics}, pages 1891--1900.

\bibitem[Audibert et~al., 2010]{audibert2010best}
Audibert, J.-Y., Bubeck, S., and Munos, R. (2010).
\newblock Best arm identification in multi-armed bandits.
\newblock {\em COLT 2010}, page~41.

\bibitem[Auer et~al., 2002]{auer2002finite}
Auer, P., Cesa-Bianchi, N., and Fischer, P. (2002).
\newblock Finite-time analysis of the multiarmed bandit problem.
\newblock {\em Machine learning}, 47(2-3):235--256.

\bibitem[Aziz et~al., 2019]{aziz2019multi}
Aziz, M., Kaufmann, E., and Riviere, M.-K. (2019).
\newblock On multi-armed bandit designs for phase {I} clinical trials.
\newblock {\em arXiv preprint arXiv:1903.07082}.

\bibitem[Berry et~al., 2004]{berry2004bayesian}
Berry, D.~A. et~al. (2004).
\newblock Bayesian statistics and the efficiency and ethics of clinical trials.
\newblock {\em Statistical Science}, 19(1):175--187.

\bibitem[Dasari et~al., 2013]{dasari2013phase}
Dasari, A., Gore, L., Messersmith, W., Diab, S., Jimeno, A., Weekes, C., Lewis,
  K., Drabkin, H., Flaig, T., and Camidge, D. (2013).
\newblock A phase {I} study of sorafenib and vorinostat in patients with
  advanced solid tumors with expanded cohorts in renal cell carcinoma and
  non-small cell lung cancer.
\newblock {\em Investigational new drugs}, 31(1):115--125.

\bibitem[Garivier and Capp{\'e}, 2011]{garivier2011kl}
Garivier, A. and Capp{\'e}, O. (2011).
\newblock The {KL-UCB} algorithm for bounded stochastic bandits and beyond.
\newblock In {\em Proceedings of the 24th annual conference on learning
  theory}, pages 359--376.

\bibitem[Hussain-Gambles et~al., 2004]{hussain2004ethnic}
Hussain-Gambles, M., Atkin, K., and Leese, B. (2004).
\newblock Why ethnic minority groups are under-represented in clinical trials:
  a review of the literature.
\newblock {\em Health \& social care in the community}, 12(5):382--388.

\bibitem[Kim et~al., 2009]{kim2009ugt1a1}
Kim, T., Sym, S., Lee, S., Ryu, M., Lee, J., Chang, H., Kim, H., Shin, J.,
  Kang, Y., and Lee, J. (2009).
\newblock A {UGT1A1} genotype-directed phase {I} study of irinotecan ({CPT-11})
  combined with fixed dose of capecitabine in patients with metastatic
  colorectal cancer ({mCRC}).
\newblock {\em Journal of Clinical Oncology}, 27(15\_suppl):2554--2554.

\bibitem[Moss et~al., 2015]{moss2015efficacy}
Moss, R.~B., Flume, P.~A., Elborn, J.~S., Cooke, J., Rowe, S.~M., McColley,
  S.~A., Rubenstein, R.~C., and Higgins, M. (2015).
\newblock Efficacy and safety of ivacaftor treatment: randomized trial in
  subjects with cystic fibrosis who have an {R117H-CFTR} mutation.
\newblock {\em The Lancet. Respiratory medicine}, 3(7):524.

\bibitem[O'Quigley et~al., 1990]{o1990continual}
O'Quigley, J., Pepe, M., and Fisher, L. (1990).
\newblock Continual reassessment method: a practical design for phase 1
  clinical trials in cancer.
\newblock {\em Biometrics}, pages 33--48.

\bibitem[Pallmann et~al., 2018]{pallmann2018adaptive}
Pallmann, P., Bedding, A.~W., Choodari-Oskooei, B., Dimairo, M., Flight, L.,
  Hampson, L.~V., Holmes, J., Mander, A.~P., Sydes, M.~R., Villar, S.~S.,
  et~al. (2018).
\newblock Adaptive designs in clinical trials: why use them, and how to run and
  report them.
\newblock {\em BMC medicine}, 16(1):29.

\bibitem[Park et~al., 2018]{park2018critical}
Park, J.~J., Thorlund, K., and Mills, E.~J. (2018).
\newblock Critical concepts in adaptive clinical trials.
\newblock {\em Clinical epidemiology}, 10:343.

\bibitem[Polyzos, 2008]{polyzos2008activity}
Polyzos, A. (2008).
\newblock Activity of {SU11248}, a multitargeted inhibitor of vascular
  endothelial growth factor receptor and platelet-derived growth factor
  receptor, in patients with metastatic renal cell carcinoma and various other
  solid tumors.
\newblock {\em The Journal of steroid biochemistry and molecular biology},
  108(3-5):261--266.

\bibitem[Riviere et~al., 2018]{riviere2018phase}
Riviere, M.-K., Yuan, Y., Jourdan, J.-H., Dubois, F., and Zohar, S. (2018).
\newblock Phase {I/II} dose-finding design for molecularly targeted agent:
  plateau determination using adaptive randomization.
\newblock {\em Statistical methods in medical research}, 27(2):466--479.

\bibitem[Storer, 1989]{storer1989design}
Storer, B.~E. (1989).
\newblock Design and analysis of phase {I} clinical trials.
\newblock {\em Biometrics}, 45(3):925--937.

\bibitem[Thompson, 1933]{thompson1933likelihood}
Thompson, W.~R. (1933).
\newblock On the likelihood that one unknown probability exceeds another in
  view of the evidence of two samples.
\newblock {\em Biometrika}, 25(3/4):285--294.

\bibitem[Varatharajah et~al., 2018]{varatharajah2018contextual}
Varatharajah, Y., Berry, B., Koyejo, S., and Iyer, R. (2018).
\newblock A contextual-bandit-based approach for informed decision-making in
  clinical trials.
\newblock {\em arXiv preprint arXiv:1809.00258}.

\bibitem[Villar et~al., 2015]{villar2015multi}
Villar, S.~S., Bowden, J., and Wason, J. (2015).
\newblock Multi-armed bandit models for the optimal design of clinical trials:
  benefits and challenges.
\newblock {\em Statistical science: a review journal of the Institute of
  Mathematical Statistics}, 30(2):199.

\bibitem[Villar and Rosenberger, 2018]{villar2018covariate}
Villar, S.~S. and Rosenberger, W.~F. (2018).
\newblock Covariate-adjusted response-adaptive randomization for multi-arm
  clinical trials using a modified forward looking gittins index rule.
\newblock {\em Biometrics}, 74(1):49--57.

\bibitem[Wages et~al., 2015]{wages2015phase}
Wages, N.~A., Read, P.~W., and Petroni, G.~R. (2015).
\newblock A phase {I/II} adaptive design for heterogeneous groups with
  application to a stereotactic body radiation therapy trial.
\newblock {\em Pharmaceutical statistics}, 14(4):302--310.

\bibitem[Wu et~al., 2015]{wu2015algorithms}
Wu, H., Srikant, R., Liu, X., and Jiang, C. (2015).
\newblock Algorithms with logarithmic or sublinear regret for constrained
  contextual bandits.
\newblock In {\em Advances in Neural Information Processing Systems}, pages
  433--441.

\bibitem[Zang et~al., 2014]{zang2014adaptive}
Zang, Y., Lee, J.~J., and Yuan, Y. (2014).
\newblock Adaptive designs for identifying optimal biological dose for
  molecularly targeted agents.
\newblock {\em Clinical Trials}, 11(3):319--327.

\bibitem[Zhou, 2015]{zhou2015survey}
Zhou, L. (2015).
\newblock A survey on contextual multi-armed bandits.
\newblock {\em arXiv preprint arXiv:1508.03326}.

\end{thebibliography}

\onecolumn

\thispagestyle{empty}
\hsize\textwidth
  \linewidth\hsize \toptitlebar {\centering
  {\Large\bfseries Supplementary Material for Contextual Constrained Learning for Dose-Finding Clinical Trials \par}}
 \bottomtitlebar

\renewcommand\thesection{\Alph{section}}
\setcounter{section}{0}

\section{PRELIMINARIES}
Before providing the proof of the theorems, we introduce some regularity assumptions on the dose-toxicity model as follows.
\begin{assumption}
\label{assumption:regularity}
There exist $C_{1,s,k}>0$, $1<\gamma_{1,s,k}$, $C_{2,s,k}$, and $0<\gamma_{2,s,k} \leq 1$ such that $|p_{s,k}(a)-p_{s,k}(a')|\geq C_{1,s,k}|a-a'|^{\gamma_{1,s,k}}$ and $|p_{s,k}(a)-p_{s,k}(a')|\leq C_{2,s,k}|a-a'|^{\gamma_{2,s,k}}$ for all $s\in\mathcal{S}$, $k\in\mathcal{K}$, and $a$, $a'\in\mathcal{A}$.
\end{assumption}
We then immediately have the following proposition.
\begin{proposition}
For $p_{s,k}(a),~\forall s\in\mathcal{S},\forall k\in\mathcal{K}$ satisfying Assumption \ref{assumption:regularity},

1. $p_{s,k}(a)$ is invertible;

2. For each $s\in\mathcal{S}$, $k\in\mathcal{K}$, and $d$, $d'\in\mathcal{P}$, we have $|p_{s,k}^{-1}(d)-p_{s,k}^{-1}(d')|\leq\bar{C}_{1,s,k}|d-d'|^{\bar{\gamma}_{1,s,k}}$, where $\bar{\gamma_{1,s,k}}=\frac{1}{\gamma_{1,s,k}}$ and $\bar{C}_{1,s,k}=C_{1,s,k}^{-\frac{1}{\gamma_{1,s,k}}}$.
\end{proposition}
For notational simplicity, we denote $C_{1,s}=\min_{k\in\mathcal{K}}C_{1,s,k}$, $C_{2,s}=\max_{k\in\mathcal{K}} C_{2,s,k}$, $\gamma_{1,s}=\max_{k\in\mathcal{K}} \gamma_{1,s,k}$, $\gamma_{2,s}=\min_{k\in\mathcal{K}} \gamma_{2,s,k}$, $\bar{\gamma}_{1,s}=\gamma_{1,s}^{-1}$, and $\bar{C}_{1,s}=C_{1,s}^{-\bar{\gamma}_{1,s}}$.

\section{PROOF OF THEOREM 1}
From the Hoeffding bound, the following upper bound of the probability is given:
\[
\mathbb{P}[\hat{a}(t)-a_s^*>\alpha_s(t)+\epsilon]\leq \exp(-2N_s(t)(\alpha_s(t)+\epsilon)^2).
\]
In addition, the difference between the MTD threshold and the expected toxicity is also bounded as
\begin{IEEEeqnarray}{ll}
p_{s,I(t)}(a_s^*)-\zeta ~& \leq p_{s,I(t)}(a_s^*)-\!\zeta\!+\!\zeta-p_{s,I(t)}(a_s^*-\alpha_s(t)) \nonumber \\
&\leq C_{2,s}|a_s^*-\hat{a}_s(t)+\alpha_s(t)|^{\gamma_{2,s}.}\nonumber
\end{IEEEeqnarray}
By rearranging the terms, we have
\begin{IEEEeqnarray}{ll}
\mathbb{P}\left[\frac{\sum_{t=1}^{N_s(T)}p_{n,I(N_s^{-1}(t))}(a_s^*)}{N_s(T)}-\zeta<C_{n,2}\epsilon^{\gamma_{n,2}}\right] ~&\geq 1-\exp(-2N_s(T)(\alpha_s(N_s(T))+\epsilon)^2)\nonumber \\
& \geq 1-\delta_s. \nonumber
\end{IEEEeqnarray}

\section{PROOF OF THEOREM 2}
\subsection{Case 1: $k_s^*\neq 0$}
We first bound the probability that the recommended dose error for subgroup $s$ occurs with C3T-Budget if $k_s^*\neq 0$.
The event that the recommended dose error occurs satisfies the following:
\[
\left\lbrace \hat{k}_s^*\neq k_s^*\right\rbrace \subseteq 
\left\lbrace p_{s,k_s^*}(\hat{a}_s(T))>\zeta \right\rbrace
\cup\left\lbrace\bar{q}_{s,k_s^*}(T)< \theta\right\rbrace
\cup\left\lbrace\max_{k\in\mathcal{K}_s}\bar{q}_{s,k}(T)\neq k_s^*\right\rbrace.
\]
Thus, the probability of the recommended dose error for subgroup $s$ can be bounded as
\[
\mathbb{P}\left[\hat{k}_s^*\neq k_s^*\right]\leq 
\mathbb{P}\left[p_{s,k_s^*}(\hat{a}_s(T))>\zeta\right]
+ \mathbb{P}\left[\bar{q}_{s,k_s^*}(T)< \theta\right]
+\mathbb{P}\left[\max_{k\in\mathcal{K}_s}\bar{q}_{s,k}(T)\neq k_s^*\right].
\]
Then, we can bound the probability by obtaining the bound for each term.

\textbf{Bound of First Term}:
The probability in the first term can be transformed and bounded as
\begin{IEEEeqnarray}{ll}
\mathbb{P}\left[\hat{a}_s(N_s(T))<p_{s,k_s^*}^{-1}(\zeta)\right]~
& \leq \mathbb{P}\left[ |a_s^* - \hat{a}_s(N_s(T))| > \Gamma_{U_s} \right] \nonumber \\
& \leq \sum_{k\in\mathcal{K}}\mathbb{P}\left[ \left| \hat{p}_{s,k}(N_s(T))-p_{s,k}(a_s^*) \right|>\left( \frac{\Gamma_{U_s}N_s(T)}{N_{s,k}(T)\bar{C}_{1,s}K} \right)^{\gamma_{1,s}} \right] \nonumber \\
&\leq\sum_{k\in\mathcal{K}}2\exp\left( -2N_{s,k}(T)\left( \frac{\Gamma_{U_s}N_s(T)}{N_{s,k}(T)\bar{C}_{1,s}K} \right)^{\gamma_{1,s}} \right) \label{eqn:B1_1} \\
&\leq 2K\exp\left( -2\left( \frac{\Gamma_{U_s}}{\bar{C}_{1,s}K} \right)^{\gamma_{1,s}}N_s(T) \right), \label{eqn:B1_2}
\end{IEEEeqnarray}
where $\Gamma_{U_s}=|a_s^*-p_{s,k}^{-1}(\zeta)|$.
The inequality in \eqref{eqn:B1_1} follows from the Hoeffding's inequality and the inequality in \eqref{eqn:B1_2} follows from the regularity assumption $\gamma_{1,s}>1$.

\textbf{Bound of Second Term}:
From the Chernoff-Hoeffding's inequality, we have
\begin{IEEEeqnarray}{ll}
\mathbb{P}\left[\bar{q}_{s,k_s^*}(T) < \theta\right]~ 
&=\mathbb{P}\left[\bar{q}_{s,k_s^*}(T)<q_{s,k_s^*} - (q_{s,k_s^*}-\theta)\right] \nonumber \\
&\leq \exp\left(-2N_{s,k_s^*}(T)\Delta_{s,\theta}^2\right) \nonumber \\
&\leq \exp\left(-\frac{8}{25}\frac{cN_s(T)\Delta_{s,\theta}^2}{\Delta_{s^*}^2}\right), \label{eqn:B2_1}
\end{IEEEeqnarray}
where $\Delta_{s^*,\theta}=|q_{s,k_s^*}-\theta|$ and 
\[
\Delta_{s^*}=\left\lbrace\begin{array}{ll}
9(\min_{k\in\mathcal{K},k\neq k_s^*}q_{s,k_s^*}-q_{s,k}),&\textrm{if }k_s^*=\max_{k\in\mathcal{K}}\{q_{s,k}\}, \\
\max_{k\in\mathcal{K}}\{q_{s,k}\}-q_{s,k_s^*}, & \textrm{otherwise}.
\end{array}\right.
\]
The inequality in \eqref{eqn:B2_1} follows from the lemmas in \citep{audibert2010best}.

\textbf{Bound of Third Term}:
The bound of the third term can be obtained by following the proof of Theorem 1 in \citep{audibert2010best}.
The third term is bounded as
\begin{equation}
\label{eqn:B3_1}
\mathbb{P}\left[\max_{k\in\mathcal{K}_s}\bar{q}_{s,k}\neq k_s^*\right] \leq 2N_s(T)K\exp\left(-\frac{2cN_s(T)}{25}\right).
\end{equation}

\textbf{Total Bound}:
From the bounds in \eqref{eqn:B1_2}, \eqref{eqn:B2_1}, and \eqref{eqn:B3_1}, we can bound 
the probability of the recommended dose error for subgroup $s$ as
\[
\mathbb{P}\left[\hat{k}_s^*\neq k_s^*\right]\leq 
\exp\left(-M_{(1a)}N_s(T)\right)
+ 2K\left(\exp\left( -M_{(1b)}N_s(T) \right)+N_s(T)\exp\left(-M_{(1c)}N_s(T)\right)\right),
\]
where
$M_{(1a)}=\frac{8}{25}\frac{c\Delta_{s,\theta}^2}{\Delta_{s^*}^2}$,
$M_{(1b)}=2\left( \frac{\Gamma_{U_s}}{\bar{C}_{1,s}K} \right)^{\gamma_{1,s}}$, and
$M_{(1c)}=\frac{2c}{25}$.

\subsection{Case 2: $k_s^*=0$}
For the case with $k_s^*=0$, we can bound the probability of the recommended dose error for subgroup $s$ similar to Case 1.
We have
\[
\left\lbrace \hat{k}_s^*\neq 0\right\rbrace \subseteq 
\bigcup_{k\in\mathcal{K}}\left\lbrace p_{s,k}(\hat{a}_s(T))\leq \zeta\right\rbrace\cup\bigcup_{k\in\mathcal{K}}\left\lbrace \bar{q}_{s,k}(T)\geq \theta\right\rbrace
\]
and
\[
\mathbb{P}\left[ \hat{k}_s^*\neq 0 \right] \leq \sum_{k\in\mathcal{K}}\mathbb{P}\left[p_{s,k}(\hat{a}_s(T))\leq \zeta\right]
+\sum_{k\in\mathcal{K}}\mathbb{P}\left[\bar{q}_{s,k}(T)\geq \theta\right].
\]
Similar to Case 1, we can bound the probability as
\[
\mathbb{P}\left[\hat{k}_s^*\neq 0\right]\leq 
K \exp\left(-M_{(2a)}N_s(T)\right)
+ 2K^2\exp\left( -M_{(2b)}N_s(T) \right),
\]
where
$M_{(1a)}=\frac{8}{25}\frac{c\Delta_{s^*\theta}^2}{\underline{\Delta}_{s^*}^2}$,
$M_{(1b)}=2\left( \frac{\bar{\Gamma}_{U_s}}{\bar{C}_{1,s}K} \right)^{\gamma_{1,s}}$,
$\bar{\Delta}_{s^*,\theta}=\max_{k\in\mathcal{K}}|\theta-q_{s,k}|$, $\underline{\Delta}_{s^*}=\max_{k\in\mathcal{K}}\{q_{s,k}\}-\min_{k\in\mathcal{K}}\{q_{s,k}\}$, and $\bar{\Gamma}_{U_s}=\max_{k\in\mathcal{K}}\Gamma_{U_s}$.

\subsection{Recommended Dose Error Bound}
Finally, we have the theorem with $M_{R1}=\max\{1+2K+N_s(T),K+2K^2\}$ and $M_{R2}=\min\{M_{(1a)},M_{(1b)},M_{(1c)},M_{(2a)},M_{(2b)}\}$.

\section{Worst-Case Reget Bound For Total Efficacy of C3T-Budget-E}
To evaluate C3T-Budget-E, we compare its performance to that of an algorithm with the complete knowledge of $q_{s,k}$'s and $p_{s,k}$'s called an oracle algorithm.
We denote the expected total cumulative efficacy achieved by the oracle algorithm by $E^*(T,B)$.
Then, the regret of C3T-Budget-E is defined as
\begin{equation}
\label{eqn:regret}
R(T,B)=E^*(T,B)-E(T,B),
\end{equation}
where $E(T,B)$ is the expected total cumulative efficacy.
Then, we provide the efficacy regret bound of C3T-Budget.
\begin{theorem}
\label{thm:regret_bound}
Given a fixed $\rho\in(0,1)$, the worst-case regret of C3T-Budget-E is bounded as
\[
R(T,B)\leq T\bar{\delta}\bar{\Delta} +q_1^*\sqrt{\rho(1-\rho)T}+M_E\log T + O(1),
\]
where
$M_E$ is a non-negative constant (provided in our supplementary material).
\end{theorem}
\begin{proof}
For the regret bound of C3T-Budget-E, we first define the optimal value of the LP problem that can be obtained by solving the LP problem with $q_s^*$'s (see (3) in our main paper) as
\[
v(\rho)=\sum_{s=1}^{\tilde{s}(\rho)}\pi_s d_s^*+\psi_{\tilde{s}(\rho)+1}(\rho)\pi_{\tilde{s}(\rho)+1}d_{\tilde{s}(\rho)+1}^*.
\]
This optimal value $v(\rho)$ can be considered the maximum expected reward in a single round with average budget $\rho$.
Thus, using $v(\rho)$, we can bound the total expected cumulative efficacy of the oracle $E^*(T,B)$ as the following lemma.
\begin{lemma}
\label{lemma:upper_bound}
\citep{wu2015algorithms} If the time-horizon and budget are given by $T$ and $B$, respectively, then we have $\hat{E}(T,B)=Tv(\rho)\geq E^*(T,B)$.
\end{lemma}
Then, the upper bound of the expected cumulative efficacy of the oracle $E^*(T,B)$ as follows.
\begin{IEEEeqnarray}{ll}
R(T,B)~&= E^*(T,B)-E(T,B) \nonumber \\
&\leq \hat{E}(T,B)-E(T,B) \nonumber \\
&= Tv(\rho)-\sum_{s\in\mathcal{S}}\sum_{k\in\mathcal{K}} q_{s,k}\mathbb{E}[N_{s,k}(T)]. \nonumber
\end{IEEEeqnarray}

From the regret using $\hat{E}(T,B)$, we can partition the regret according to the source of regret as follows:
\begin{IEEEeqnarray}{ll}
R(T,B)~&= Tv(\rho)-\sum_{s\in\mathcal{S}}\sum_{k\in\mathcal{K}} 
q_{s,k}\mathbb{E}[N_{s,k}(T)] \nonumber \\
& = \sum_{s\in\mathcal{S}}\sum_{k\in\mathcal{K}} 
\Delta_{s,k}^{(s)}\mathbb{E}[N_{s,k}(T)]
+Tv(\rho)-\sum_{s\in\mathcal{S}}q_s^*\mathbb{E}[N_s(T)] \nonumber \\
& = \underbrace{\sum_{s\in\mathcal{S}}\sum_{k\in\mathcal{K}} 
\Delta_{s,k}^{(s)}\mathbb{E}[N_{s,k}(T)]}_{=R^{(1)}(T,B)}  + \underbrace{\sum_{t=1}^{T} \mathbb{E}\left[ v(\rho)-\sum_{s\in\mathcal{S}}\hat{\psi}_s(\rho(t))\pi_s q_s^* \right]}_{=R^{(2)}(T,B)}. \nonumber
\end{IEEEeqnarray}
Recall that $\Delta_{s,k}^{(s')}$ is the difference between the optimal expected efficacy of subgroup $s'$ and the expected efficacy of subgroup $s$ with dose $k$, $\Delta_{s,k}^{(s')}=q_{s'}-q_{s,k}$.
The decomposed regret $R^{(1)}$ represents the regret due to taking suboptimal doses
and the other decomposed regret $R^{(2)}$ represents the regret due to ordering errors in subgroups.
It is worth noting that in $R^{(2)}$, it is supposed that the optimal doses are chosen.
Finally, the regret of C3T-Budget-E is bounded as
\[
R(T,B)\leq  R^{(1)}(T,B)+R^{(2)}(T,B).
\]
Then, we can bound the regret of C3T-Budget by obtaining the bound of each regret.

\subsection{Bound of $R^{(1)}$}
We first bound the first part of the regret, $R^{(1)}$.
In C3T-Budget-E, the set of candidate recommended doses is constructed in each round and the dose is chosen among the doses in the set.
Thus, taking the suboptimal doses can occur due to not only the inaccurate estimation of the efficacy but also the inaccurate estimation of the toxicity.
To reflect this, we decompose the regret $R^{(1)}$ into two parts according to whether the optimal dose is included in the set of the candidate doses or not as follows.
\[
R^{(1)}(T,B)=\underbrace{\sum_{t=1}^{T} \sum_{s\in\mathcal{S}} \mathbb{I}\{H(t)=s\} \mathbb{P}[k_s^*\notin \mathcal{K}_{s}(t)]\bar{\Delta}_s}_{=R^{(1a)}(T,B)}+\underbrace{\sum_{t=1}^{T} \sum_{s\in\mathcal{S}} \mathbb{I}\{H(t)=s\}\mathbb{P}[k_s^*\in\mathcal{K}_{s}(t)]R_{s,2}(t)}_{=R^{(1b)}(T,B)},
\]
where $\bar{\Delta}_s=\max_{k\in\mathcal{K}}\Delta_{s,k}^{(s)}$.

\textbf{Bound of} $R^{(1a)}(T,B)$:
We bound the regret $R^{(1a)}$.
Since the event $\{k_s^*\notin\mathcal{K}_s(t)\}$ can be bounded by
$\{p_{s,k_s^*}(\hat{a}_s(t)-\alpha_s(t))>\zeta\}\cup\{\hat{q}_{s,k_s^*}(t)<\theta \}$,
we can bound the regret $R^{(1a)}$ as
\begin{IEEEeqnarray}{ll}
R^{(1a)}(T,B)~&\leq\sum_{t=1}^{T} \sum_{s\in\mathcal{S}} \mathbb{I}\{H(t)=s\} \left(\mathbb{P}[p_{s,k_s^*}(\hat{a}_s(t)-\alpha_s(t))>\zeta]+\mathbb{P}[\hat{q}_{s,k_s^*}(t)<\theta]\right)\bar{\Delta}_s \nonumber \\
&=\underbrace{\sum_{t=1}^{T} \sum_{s\in\mathcal{S}} \mathbb{I}\{H(t)=s\} \mathbb{P}[p_{s,k_s^*}(\hat{a}_s(t)-\alpha_s(t))>\zeta]\bar{\Delta}_s}_{=R^{(1a-1)}(T,B)}
+ \underbrace{\sum_{t=1}^{T} \sum_{s\in\mathcal{S}} \mathbb{I}\{H(t)=s\} \mathbb{P}[\hat{q}_{s,k_s^*}(t)<\theta]\bar{\Delta}_s}_{=R^{(1a-2)}(T,B)}. \nonumber
\end{IEEEeqnarray}

We first bound $R^{(1a-1)}(T,B)$.
In the following lemma, we show that the safe dose for each subgroup (i.e., the toxicities of the dose levels are below the MTD threshold) are included in the set of the candidate doses with high probability.
\begin{lemma}
For each subgroup $s$, $\mathbb{P}[p_{s,k}(\hat{a}_s(t)+\alpha_s(t))>\zeta]\leq \delta_s$, for any $p_{s,k}(a_s^*)\leq\zeta$.
\end{lemma}
\begin{proof}
We have
\begin{IEEEeqnarray}{ll}
\mathbb{P}[\hat{a}_s(t)+\alpha_s(t) < a_s^*] 
~&=\mathbb{P}[a_s^*-\hat{a}_s(t)>\alpha_s(t)] \nonumber \\
&\leq \sum_{k\in\mathcal{K}}\mathbb{P}\left[ |\hat{p}_{s,k}(t)-p_{s,k}(a_s^*)|>\left(\frac{\alpha_s(t)N_{s}(t)}{N_{s,k}(t)\bar{C}_{s,1}K}\right)^{\gamma_{s,1}}\right] \nonumber \\
&\leq \sum_{k\in\mathcal{K}}2\exp\left(-2N_{s,k}(t)\left(\frac{\alpha_s(t)N_{s}(t)}{N_{s,k}(t)\bar{C}_{s,1}K}\right)^{2\gamma_{s,1}}\right) \label{eqn:C1_1} \\
&\leq 2K\exp\left(-\left(\frac{\alpha_s(t)}{\bar{C}_{s,1}K}\right)^{2\gamma_{s,1}}N_{s}(t)\right)=\delta_s \nonumber
\end{IEEEeqnarray}
The inequality in \eqref{eqn:C1_1} follows from the Hoeffding's inequality.
\end{proof}
From this lemma, the probability that the event $\{p_{s,k_s^*}(\hat{a}_s(t)-\alpha_s(t))>\zeta\}$ occurs is bounded by $\delta_s$ since the set of the candidate doses for subgroup $s$ is constructed by $\{k\in\mathcal{K}:p_{s,k}(\hat{a}_s(t)+\alpha_s(t))\leq \zeta\}$ in C3T-Budget-E.
Then, the regret $R^{(1a-1)}$ can be simply bounded as
\begin{IEEEeqnarray}{ll}
R^{(1a-1)}(T,B)~&\leq \sum_{s\in\mathcal{S}}N_s(T)\delta_s\bar{\Delta}_s \nonumber \\
&\leq T\bar{\delta}\bar{\Delta} \nonumber,
\end{IEEEeqnarray}
where $\bar{\delta}=\max_{s\in\mathcal{S}}\delta_s$ and $\bar{\Delta}=\max_{s\in\mathcal{S}}\bar{\Delta}_s$.

We bound the regret $R^{(1a-2)}(T,B)$.
For the minimum efficacy threshold, we have the following lemma.
\begin{lemma}
Let For each subgroup $s$, $\mathbb{P}[\hat{q}_{s,k_s^*}(t)<\theta]\leq N_s(t)^{-2c}$.
\end{lemma}
\begin{proof}
We have
\begin{IEEEeqnarray}{ll}
\mathbb{P}[\hat{q}_{s,k_s^*}(t)<\theta] 
~&\leq \mathbb{P}\left[\bar{q}_{s,k_s^*}(t) < q_{s,k_s^*}-\sqrt{\frac{c\log N_s(t)}{N_{s,k_s^*}(t)}}\right] \nonumber \\
&\leq N_s(t)^{-2c} \nonumber
\end{IEEEeqnarray}
The first inequality follows from the fact $q_{s,k_s^*}\geq\theta$ and the second inequality follows from the Chernoff-Hoeffding inequality.
\end{proof}
Then, for $c\geq \frac{1}{2}$, the regret $R^{(1a-2)}(T,B)$ is bounded as
\begin{IEEEeqnarray}{ll}
\sum_{t=1}^{T} \sum_{s\in\mathcal{S}} \mathbb{I}\{H(t)=s\} \mathbb{P}[\hat{q}_{s,k_s^*}(t)<\theta]\bar{\Delta}_s ~
&\leq \sum_{t=1}^{T} \sum_{s\in\mathcal{S}} \mathbb{I}\{H(t)=s\} N_s(t)^{-2c}\bar{\Delta}_s \nonumber \\
&\leq S\sum_{t=1}^{T} t^{-2c} \bar{\Delta} \nonumber \\
&\leq S\log(T)\bar{\Delta}
\end{IEEEeqnarray}
In summary, we have the bound of $R^{(1a)}(T,B)$ as following.
\[
R^{(1a)}(T,B) \leq (T\bar{\delta}+S\log T)\bar{\Delta}
\]

\textbf{Bound of} $R^{(1b)}(T,B)$:We now bound the regret $R^{(1b)}$.
In this case, the optimal dose is included in the set of the candidate doses.
Then, the error occurs when the suboptimal doses are chosen in the set due to the inaccurate parameter estimation $\hat{a}_s$ and the inaccurate efficacy estimation.
To bound this, in the following lemma, we show that the unsafe dose for each subgroup (i.e., the toxicities of the dose levels exceed the MTD threshold) are excluded from the set of the candidate doses with high probability.
\begin{lemma}
Let $\Gamma_s=\min_{k\in\mathcal{K}}|a_s^*-p_{s,k}^{-1}(\zeta)|$ and
\[
\alpha_s(t)=\bar{C}_{s,1}K\left(\frac{\log\frac{2K}{\delta_s}}{2N_s(t)}\right)^{\frac{\bar{\gamma}_{s,1}}{2}}.
\]
For each subgroup $s$, if $N_s(t)\geq t_s^{(1b)}= \frac{1}{2}\left(\frac{\bar{C}_{s,1}K}{\Gamma_s - \epsilon}\right)^{2\gamma_{s,1}}\log{\frac{2K}{\delta_s}}$, then we have
$
\mathbb{P}[p_{s,k}(\hat{a}_s(t)+\alpha_s(t))\leq \zeta]\leq \exp(-2N_s(t)\epsilon^2),
$
for any $p_{s,k}(a_s^*)>\zeta$.
\end{lemma}
\begin{proof}
From the Hoeffding's inequality, we have
\[
\alpha_s(t)\leq p_{s,k}^{-1}(\zeta)-a_s^*-\epsilon=\Delta_{s,k}-\epsilon,
\]
where $\Delta_{s,k}=|a_s^*-p_{s,k}^{-1}(\zeta)|$.
With the definition of $\alpha_s(t)$, we can conclude the lemma.
\end{proof}
Let $N_s^{-1}(\tau)=\min_{t}\{t:N_s(t)=\tau\}$ which represents the round index at which the $\tau$-th patient of subgroup $s$ arrives.
Then, we can bound the regret $R^{(1b)}$ as
\begin{IEEEeqnarray}{lll}
R^{(1b)}~&\leq ~& \sum_{s\in\mathcal{S}} \left[ t_s^{(1b)}+ (K-U_s)\sum_{t=1}^{N_s(T)}\exp(-2t\epsilon^2) + \sum_{t=t_{s}+1}^{N_s(T)}\sum_{k:p_{s,k}(a_s^*)\leq\zeta}\mathbb{I}\{I(N_s^{-1}(t))=k\}\right] \nonumber \\
&\leq & \sum_{s\in\mathcal{S}} \left[ t_s^{(1b)}+ \frac{K-U_s}{2\epsilon^2}+\sum_{k:p_{s,k}(a_s^*)\leq\zeta}\frac{c\log T}{\Delta_{s,k}^{(s)}}\right] \nonumber \\
&\leq & \bar{t}^{(1b)}+ \frac{K-\underline{U}}{2\epsilon^2}+\sum_{s\in\mathcal{S}} \sum_{k:p_{s,k}(a_s^*)\leq\zeta}\frac{c\log T}{\Delta_{s,k}^{(s)}}, \nonumber
\end{IEEEeqnarray}
where $\bar{t}^{(1b)}=\max_{s\in\mathcal{S}}t_s^{(1b)}$ and $\underline{U}=\min_{s\in\mathcal{S}}U_s$.

\textbf{Bound of} $R^{(1)}(T,B)$:
Finally, from the bounds of $R^{(1a)}$ and $R^{(1b)}$, we have the regret bound of $R^{(1)}$ as following:
\begin{equation}
\label{eqn:regret_1}
R^{(1)}(T,B)\leq (T\bar{\delta}+S\log T)\bar{\Delta} +\sum_{s\in\mathcal{S}} \sum_{k:p_{s,k}(a_s^*)\leq\zeta}\frac{c\log T}{\Delta_{s,k}^{(s)}} +O(1).
\end{equation}

\subsection{Bound of $R^{(2)}$}
We now bound the second part of the regret, $R^{(2)}$.
Recall that the regret $R$ is decomposed into two parts: the regret due to taking suboptimal doses $R^{(1)}$ and the regret due to ordering errors in subgroups $R^{(2)}$.
Thus, in here, we do not have to consider the suboptimal doses and consider the ordering errors only.

In \cite{wu2015algorithms}, the regret due to the ordering error is analyzed.
Compared with the case that is analyzed, we additionally consider the safety constraint.
However, we can follow the analysis on the regret due to the ordering error in \cite{wu2015algorithms} for the bound of $R^{(2)}$ since the safety constraint reduces the ordering errors by excluding the unsafe doses which is not the optimal doses.
Thus, we can provide the regret bound of $R^{(2)}$ by using the analysis.

Before providing the regret bound, we define some boundary cases according to $\rho$ and $\eta_s$'s since the bound depends on them.
We first define a non-boundary case for a given fixed $\rho\in(0,1)$ as a case in which $\rho\neq\eta_s$ for any $s\in\mathcal{S}$,
and define a boundary case for a given fixed $\rho\in(0,1)$ as a case in which $\rho=\eta_s$ for some $s\in\mathcal{S}$.
Then, by applying the analysis on our algorithm, we have the regret bounds on the following lemma.
\begin{lemma}
\label{lemma:regret_2}
\citep{wu2015algorithms} Given a fixed $\rho\in(0,1)$, the regret $R^{(2)}(T,B)$ is bounded as follows:

(1) For the non-boundary case,
\[
R^{(2)}(T,B)\leq [\bar{q}^*+v(\rho)]M^{(2)}_{nb}\log T + O(1)
\]
(2) For the boundary case,
\[
R^{(2)}(T,B)\leq q_1^*\sqrt{\rho(1-\rho)T}+M^{(2)}_b\log T + O(1)
\]

where $\bar{q}^*=\sum_{s\in\mathcal{S}}\pi_s q_{s}^*$,
\[
M^{(2)}_{nb}=\sum_{s=1}^{\tilde{s}(\rho)}\sum_{k\in\mathcal{K}}\frac{27}{2g^{nb}_{\tilde{s}(\rho)+1}\left[\Delta^{(s)}_{\tilde{s}(\rho)+1,k}\right]^2} +\sum_{s=\tilde{s}+2}^{S}\sum_{k\in\mathcal{K}}\frac{27}{2g^{nb}_{s}\left[\Delta^{(\tilde{s}+1)}_{s,k}\right]^2}
+2SK,
\]
\[
M^{(2)}_{b}=\sum_{s=1}^{\tilde{s}(\rho)-1}\sum_{k\in\mathcal{K}}\frac{27}{2g^b_{\tilde{s}(\rho)}\left[\Delta^{(s)}_{\tilde{s}(\rho),k}\right]^2} +\sum_{s=\tilde{s}+1}^{S}\sum_{k\in\mathcal{K}}\frac{27}{2g^b_{s}\left[\Delta^{(\tilde{s})}_{s,k}\right]^2}
+2SK,
\]
\[
g^{nb}_s=\min\left\lbrace \pi_s,\frac{1}{2}(\rho-\eta_{\tilde{s}(\rho)}),\frac{1}{2}(\eta_{\tilde{s}(\rho)+1}-\rho)\right\rbrace,
\]
\[
\textrm{ and } g^b_s=\min\left\lbrace\pi_s,\frac{1}{2}(\rho-\eta_{\tilde{s}(\rho)-1}),\frac{1}{2}(\eta_{\tilde{s}(\rho)+1}-\rho)\right\rbrace.
\]
\end{lemma}

\subsection{Regret Bound of $R(T,B)$}
Form Lemma \ref{lemma:regret_2}, we can see that the boundary case has a worse bound $O(\sqrt{T}\log T)$ than the non-boundary case $O(\log T)$.
Hence, with \eqref{eqn:regret_1} and Lemma \ref{lemma:regret_2}, we have the worst-case regret bound of C3T-Budget-E in the theorem with
\[
M=\sum_{s\in\mathcal{S}} \sum_{k:p_{s,k}(a_s^*)\leq\zeta}\frac{c}{\Delta_{s,k}^{(s)}}+\sum_{s=1}^{\tilde{s}(\rho)-1}\sum_{k\in\mathcal{K}}\frac{27}{2g^b_{\tilde{s}(\rho)}\left[\Delta^{(s)}_{\tilde{s}(\rho),k}\right]^2} +\sum_{s=\tilde{s}+1}^{S}\sum_{k\in\mathcal{K}}\frac{27}{2g^b_{s}\left[\Delta^{(\tilde{s})}_{s,k}\right]^2}
+2S(K+\bar{\Delta}),
\]
\[
\textrm{ and } g^b_s=\min\left\lbrace\pi_s,\frac{1}{2}(\rho-\eta_{\tilde{s}(\rho)-1}),\frac{1}{2}(\eta_{\tilde{s}(\rho)+1}-\rho)\right\rbrace.
\]
\end{proof}

\section{DESCRIPTION OF C3T-Budget-E}
For C3T-Budget-E, we consider the following formulation:
\begin{IEEEeqnarray}{cl}
\maximize ~& E_\Pi(T,B) \nonumber \\
\subjecto ~&  \mathbb{P}\left[S_{\Pi,s}(T,B)\leq\zeta\right] \geq 1-\delta_s,~\forall s\in\mathcal{S} \nonumber  \\
& \textstyle\sum_{t=1}^{T} Z_t\leq B. \nonumber
\end{IEEEeqnarray}
where we have simply substituted the objective function $D_\Pi(T,B)$ in the limited-budget C3T problem (See (2) in our main paper) with $E_\Pi(T,B)$.
With this formulation, the agent tries to achieve high efficacies (rather than low dose recommendation error), which results in focusing on subgroups with high efficacies.
We now provide a detailed description of C3T-Budget-E to solve the above problem.
\begin{algorithm}[ht]
\caption{C3T-Budget-E}
\label{alg:algorithm2}
\begin{algorithmic}[1]
%\footnotesize
\State \textbf{Input:} Time-horizon $T$, budget $B$, and subgroup arrival distributions $\pi_s$'s
\State \textbf{Initialize:} $\tau=T$, $b=B$, $t=1$
\While{$t\leq T$}
\State $\hat{a}_s(t)\leftarrow \frac{\sum_{k=1}^{K} \hat{a}_{s,k}(t-1)N_{s,k}(t-1)}{N_{s}(t-1)},\forall s\in\mathcal{S}$
%\State $\mathcal{D}_{s}(t)=\{u_k\in\mathcal{D}:p_{s,k}(\hat{a}_s(t)+\alpha_s(t))\leq\zeta\},\forall s\in\mathcal{S}$
\State $\mathcal{K}_{s}(t)=\{k\in\mathcal{K}:p_{s,k}(\hat{a}_s(t)+\alpha_s(t))\leq\zeta\},\forall s\in\mathcal{S}$
\If{$b > 0$}
\If{$N_{H(t)}(t) \leq K$}
\State{Sample each dose once $I(t)=N_{H(t)}(t)$}
\Else
\State $k_s^*(t)\leftarrow\argmax_{k\in\mathcal{K}_{s}} \hat{q}_{s,k}(t),\forall s\in\mathcal{S}$
\State $\hat{q}_s^*(t)\leftarrow\max_{k\in\mathcal{K}_s}\hat{q}_{s,k}(t),\forall s\in\mathcal{S}$
\parState{Obtain $\hat{\bpsi}\left(b/\tau\right)$'s by solving the LP problem (See (3) in our main paper) with ordered $\hat{q}_s^*(t)$'s}
\parState{Allocate dose $
I(t)\!=\!\left\lbrace\begin{array}{ll}
\!\!k_{H(t)}^*(t),& \!\!\!\textrm{with probability }\hat{\psi}_{H(t)}(b/\tau), \\
\!\!0,&  \!\!\!\textrm{otherwise}.
\end{array}\right.
$}
\EndIf
\EndIf
%\State Update $\tau$, $b$ and parameters in ICT-Ss
\State Observe the efficacy $X_t$ and toxicity $Y_t$
%\For{$s=1$ \textbf{to} $S$}
%\State $\hat{a}_s(t)\leftarrow \frac{\sum_{k=1}^{K} \hat{a}_{s,k}(t-1)N_{s,k}(t-1)}{N_{s}(t-1)}$
%\State $\mathcal{D}_{s}(t)=\{u_k\in\mathcal{D}:p_{s,k}(\hat{a}_s(t)+\alpha_s(t))\leq\zeta\}$
%\State $k_s^*(t)\leftarrow\argmax_{u_k\in\mathcal{D}_{s}} \hat{q}_{s,k}(t)$
%\State $\hat{q}_s^*(t)\leftarrow\hat{q}_{s,k_s^*(t)}(t)$
%\EndFor
\State Update $\tau$, $b$, $N_s(t)$, $N_{s,k}(t)$, $\bar{q}_{s,k}(t)$, $\bar{p}_{s,k}(t)$
\parState{$\hat{a}_{s,k}(t)\leftarrow\argmin_a|p_{s,I(t)}(a)-\bar{p}_{s,I(t)}(t)|,\forall s\in\mathcal{S},\forall k\in\mathcal{K}$}
\State $t\leftarrow t+1$
\EndWhile
%\parState{\textbf{Output:} Dose recommendations
%$\hat{d}_s(T,B)=\argmax_{k:p_{s,k}(\hat{a}_s(T))\leq\zeta}\bar{q}_{s,k}(T)$}
% for each subgroup by using the corresponding ICT-S
\end{algorithmic}
\end{algorithm}

\section{DESCRIPTION OF BASELINE ALGORITHMS}
To evaluate the performance of C3T-Budget and C3T-Budget-R, we implement the following baseline algorithms:
\begin{itemize}
\item \textbf{Contextual UCB (C-UCB)} \citep{auer2002finite,varatharajah2018contextual}:
C-UCB is an extended version of a traditional UCB algorithm in \cite{auer2002finite} for a contextual bandits.
We implement it by running $S$ instances of the tradition UCB algorithm for each subgroup as introduced in \cite{zhou2015survey}.
In the algorithm, the safety and budget constraints are not considered.
The algorithm updates the empirical expected efficacy of each dose for each subgroup and its confidence bound.
It also updates the empirical toxicities, but they are used only for the recommendation.
In each round, the algorithm chooses the dose having the highest UCB index for the subgroup arrived in the round.
At the end of trial, it recommends a dose for each subgroup as: $\bar{d}_{s}=\argmax_{k:\hat{p}_{s,k}\leq\zeta, \bar{q}_{s,k}(t)\geq \theta} \bar{q}_{s,k}$,
where $\hat{p}_{s,k}$ and $\hat{q}_{s,k}$ are the empirical expected toxicity and efficacy of dose $k$ for subgroup $s$.

\item \textbf{Contextual KL-UCB (C-KL-UCB)} \citep{garivier2011kl,varatharajah2018contextual}:
C-KL-UCB is an extended version of a KL-UCB algorithm in \cite{garivier2011kl} for a contextual bandits.
Similar to C-UCB, we implement it by using $S$ instances of KL-UCB.
The algorithm is same with C-UCB except for using the KL-UCB index instead of the UCB index, which is given by
\[
\hat{q}_{s,k}(t)=\sup\{q\geq \bar{q}_{s,k}(t) : N_{s,k}(t-1)I(\bar q_{s,k}),q) \leq \log N_s(t) + \log\log N_s(t)\}, \nonumber
\]
where $I(p,q)=p\log\frac{p}{q}+(1-p)\log\frac{1-p}{1-q}$ is the Kullback-Leibler divergence.

\item \textbf{Contextual independent Thompson sampling (C-Indep-TS)} \citep{aziz2019multi}:
C-Indep-TS is an extended version of an Indep-TS algorithm in \cite{aziz2019multi,thompson1933likelihood} for a contextual bandits.
Similar to other extended algorithms, we implement it by using $S$ instances of Indep-TS.
In Indep-TS, a Bayesian approach is used to estimate the efficacy and toxicity as follows:
\[
\hat{q}_{s,k}(t)\sim \textrm{Beta}(\alpha^q_{s,k}(t),\beta^q_{s,k}(t)) \textrm{ and } \hat{p}_{s,k}(t)\sim \textrm{Beta}(\alpha^p_{s,k}(t),\beta^p_{s,k}(t)),
\]
where
$\alpha^q_{s,k}(t)=X_{s,k}(t)+1$,
$\beta_{s,k}^q(t)=N_{s,k}(t)-X_{s,k}(t)+1$,
$X_{s,k}(t)=\sum_{\tau=1}^{t}\mathbb{I}\{H(\tau)=s,I(\tau)=k\}X(\tau)$,
$\alpha^p_{s,k}(t)=Y_{s,k}(t)+1$,
$\beta_{s,k}^p(t)=N_{s,k}(t)-Y_{s,k}(t)+1$, and
$Y_{s,k}(t)=\sum_{\tau=1}^{t}\mathbb{I}\{H(\tau)=s,I(\tau)=k\}Y(\tau)$.
In each round $t$, the efficacy and toxicity for subgroup $H(t)$ is realized based on the posterior distribution as above equation, and then, the dose $k$ that has the maximum realized efficacy $\hat{q}_{H(t),k}$ is chosen.
At the end of the trial, it recommends a dose for each subgroup as: $\bar{d}_s=\argmax_{k:\hat{p}_{s,k}(t)\leq \zeta, \hat{q}_{s,k}(t)\geq \theta}\hat{q}_{s,k}(t)$.

\item \textbf{Contextual 3+3 (C-3+3)} \citep{storer1989design}:
C-3+3 is an extended version of a 3+3 clinical trial design in \cite{storer1989design} for a contextual model.
Similar to other extended algorithms, we implement it by using $S$ instances of 3+3 design.
In 3+3 design, for each subgroup, the lowest dose is treated to 3 patients.
Then, it observes the toxicity of the patients.
If the agent observes the toxicity from none of patients, then the next dose is treated to another 3 patients.
If the agent observes only one toxicity,  the same dose is treated to additional 3 patients.
If the agent still observes only one toxicity among the 6 patients, then the next dose is treated to another 3 patients.
Otherwise, the trial is stopped and the dose treated before stopping is recommended.
If the instance of 3+3 design for a subgroup is stopped once, then the patients in the subgroup are skipped.

%\item \textbf{Contextual CRM (C-CRM)} ():
\end{itemize}

\end{document}